\documentclass[10pt,
]{article}

\usepackage[
	a4paper, scale=0.77,
]{geometry}

\usepackage{ifpdf}
\ifpdf\fi
\usepackage{amsmath}
\usepackage{amsfonts}
\usepackage{amsthm}
\usepackage{amssymb}
\usepackage{latexsym}
\usepackage{fixmath}% for slanted capital Greek letters: Gamma Delta Theta Lambda Xi Pi Sigma Upsilon Phi Psi Omega
\usepackage{mathdots}

\usepackage{xcolor}
\usepackage{graphicx}
\usepackage{subfig}

\usepackage{siunitx}

\ifpdf
\usepackage{hyperref}           % this causes trouble for Yap
\else
\usepackage[hypertex]{hyperref}  % after some googling, I found this that worked well with Yap
\fi
\hypersetup{colorlinks}

\usepackage[capitalise]{cleveref}
\crefname{assumption}{assumption}{assumptions}
\crefname{figure}{figure}{figures}
\crefname{equation}{}{}
\crefname{subsection}{subsection}{subsections}

\usepackage[shortlabels]{enumitem}
\setlist[enumerate,2]{label=(\alph*),ref=\theenumi(\alph*)}
\setlist[enumerate,3]{label=\roman*.,ref=\theenumii\roman*}

\usepackage{algorithm,algorithmic}

\usepackage{mathtools}
\providecommand\given{\:\vert\:}
\DeclarePairedDelimiterXPP\set[1]{}{\{}{\}}{}{
\renewcommand\given{\nonscript\:\delimsize\vert\nonscript\:\mathopen{}}
#1}
\DeclarePairedDelimiter\abs{|}{|}
\DeclarePairedDelimiter\N{\|}{\|}

\DeclarePairedDelimiter\floor{\lfloor}{\rfloor}

\usepackage{accents}
\newcommand{\ubar}[1]{\underaccent{\bar}{#1}}
\newcommand{\ol}[2][1]{{}\mkern#1mu\bar{\mkern-#1mu#2}}
\newcommand{\ul}[2][1]{{}\mkern#1mu\ubar{\mkern-#1mu#2}}
\def\wtd{\widetilde}
\def\what{\widehat}
\def\dps{\displaystyle}
\def\scs{\scriptstyle}

\def\ee{\mathrm{e}}

\DeclareMathOperator{\diag}{diag}

\DeclareMathOperator{\subspan}{span}
\DeclareMathOperator{\trace}{trace}

\DeclareMathOperator{\UI}{ui}
\DeclareMathOperator{\F}{F}

\DeclareMathOperator{\T}{T}

\DeclareMathOperator{\rmc}{c}

\newtheorem{theorem}{Theorem}[section]
\newtheorem{lemma}{Lemma}[section]

\theoremstyle{definition}
\newtheorem{definition}{Definition}[section]

\newtheorem{assumption}{Assumption}[section]

\numberwithin{algorithm}{section}
\numberwithin{equation}{section}
\numberwithin{figure}{section}
\numberwithin{table}{section}
\allowdisplaybreaks
\usepackage[normalem]{ulem}

\usepackage{blkarray}
\usepackage{tikz}
\usetikzlibrary{graphs}

\usepackage{mathrsfs}
\def\scrT{\mathscr{T}}

\def\bX{\boldsymbol{X}}
\def\bY{\boldsymbol{Y}}

\def\eventA{\mathbb{A}}
\def\eventH{\mathbb{H}}
\def\eventM{\mathbb{M}}
\def\eventQ{\mathbb{Q}}
\def\eventT{\mathbb{T}}

\def\overevent{;\:}
\def\new{^+}
\def\fil{\mathbb{F}}
\def\opL{\mathcal{L}}
\def\opI{\mathcal{I}}

\DeclareMathOperator{\sphere}{\mathbb{S}}
\DeclareMathOperator{\opE}{\mathrm{E}}
\DeclareMathOperator{\opprob}{\mathrm{P}}
\newcommand\E[2][*]{\opE\set#1{#2}}
\newcommand\prob[2][*]{\opprob\set#1{#2}}
\newcommand\Nout[2][*]{N_{\mathrm{out}}\set#1{#2}}
\newcommand\Nin[2][*]{N_{\mathrm{in}}\set#1{#2}}
\newcommand\Nqb[2][]{N_{\mathrm{qb}}\set#1{#2}}
\DeclareMathOperator{\Sum}{sum}

\DeclareMathOperator{\OO}{O}

\DeclareMathOperator{\var}{var}
\DeclareMathOperator{\cov}{cov}

\newcommand\ind[1]{\mathbf{1}_{#1}}
\newcommand\varc[1]{\var_{\circ}\!\left(#1\right)}
\newcommand\covc[1]{\cov_{\circ}\!\left(#1\right)}

\title{On the Optimality of the Oja's Algorithm for Online PCA}

\author{
Xin Liang\thanks{%
	Yau Mathematical Sciences Center, Tsinghua University, Beijing 100084, China.
E-mail: {\tt liangxinslm@tsinghua.edu.cn}.
Supported by NSFC-11901340.
}
}
\date{\today}

\begin{document}
\maketitle
\begin{abstract}
	In this paper we analyze the behavior of the Oja's algorithm for online/streaming principal component subspace estimation.
	It is proved that with high probability it performs an efficient, gap-free, global convergence rate to approximate an principal component subspace for any sub-Gaussian distribution.
	Moreover, it is the first time to show that the convergence rate, namely the upper bound of the approximation, exactly matches the lower bound of an approximation obtained by the offline/classical PCA up to a constant factor.
\end{abstract}

\smallskip
{\bf Key words.} Principal component analysis, Stochastic approximation, High-dimensional data, Oja's algorithm

\smallskip
{\bf AMS subject classifications}.
62H25, 68W27, 65F15
\section{Introduction}\label{sec:introduction}
Principal component analysis (PCA) introduced by Pearson~\cite{pearson1901lines} and Hotelling~\cite{hotelling1933analysis} is
one of the most well-known and popular methods for dimensional reduction in statistics, machine learning, and data science.
The goal of PCA is to find out a low-dimensional linear subspace that is closest to a centered random vector in a high-dimensional subspace in the mean squared sense through finite independent and identically distributed (i.i.d.) samples of the random vector.
Theoretically, given a random vector $\bX\in \mathbb{R}^d$ satisfying 
	$\E{\bX}=0,\E{\bX\bX^{\T}}=\Sigma$,
PCA looks for a subspace $\mathcal{U}_*$ with $\dim \mathcal{U}_*=p<d$, such that
\begin{equation}\label{eq:PCA-opt}
		\mathcal{U}_*=\arg\min_{\dim\mathcal{U}=p}\E{\N{(I_d-\Pi_{\mathcal{U}})\bX}_2^2},
\end{equation}
where $I_d$ is the identical mapping, or equivalently the $d\times d$ identity matrix, and $\Pi_{\mathcal{U}}$ is the orthogonal projector onto $\mathcal{U}$.
Let $\Sigma=U\Lambda U^{\T}$ be the spectral decomposition of $\Sigma$,
where
\begin{equation}\label{eq:eigD-convar}
\Sigma=U\Lambda U^{\T}
\quad\text{with}\quad
U=[u_1,u_2,\ldots,u_d],\,\,
\Lambda=\diag(\lambda_1,\dots,\lambda_d),
\end{equation}
If $\lambda_p>\lambda_{p+1}$,
then the unique solution to the optimization problem \cref{eq:PCA-opt}, namely the $p$-dimensional principal
subspace of $\Sigma$, is $\mathcal{U}_*=\subspan(u_1,\dots,u_p)$, the subspace spanned by $u_1,\dots,u_p$.

In practice, the covariance matrix $\Sigma$ is difficult, if not impossible, to obtain,
and people have to use samples to approximate $\mathcal{U}_*$.
The classical/offline PCA use the spectral decomposition of the empirical covariance matrix $\what \Sigma=\frac{1}{n}\sum_{i=1}^n X^{(i)}(X^{(i)})^{\T}$. 
There $\what{\mathcal{U}}_*=\subspan{\what u_1,\dots,\what u_p}$ is used to approximate $\mathcal{U}_*$, where $\what u_i$ are corresponding eigenvectors of $\what \Sigma$.
Vu and Lei~\cite[Theorem~3.1]{vuL2013minimax} proved that
%if $p(d-p)\frac{\sigma_*^2}{n}$ is bounded, then
\begin{equation}\label{eq:minimax-bound}
	\inf_{\dim\wtd{\mathcal{U}}_*=p}
	\sup_{\bX\in\mathcal{P}_0(\sigma_*^2,d)}
	\E{\N{\sin\Theta(\wtd{\mathcal{U}}_*,{\mathcal{U}}_*)}_{\F}^2}
	\ge cp(d-p)\frac{\sigma_*^2}{n}
	\ge c\frac{\lambda_1\lambda_{p+1}}{(\lambda_p-\lambda_{p+1})^2}\frac{p(d-p)}{n},
\end{equation}
where $c>0$ is an absolute constant, and $\mathcal{P}_0(\sigma_*^2,d)$ is the set of all $d$-dimensional
sub-Gaussian distributions for which the eigenvalues of the covariance matrix satisfy
$\frac{\lambda_1\lambda_{p+1}}{(\lambda_p-\lambda_{p+1})^2}\le\sigma_*^2$.
Note that $\frac{\lambda_1\lambda_{p+1}}{(\lambda_p-\lambda_{p+1})^2}$
is the effective noise variance.

Due to the practical requirement that only limited memory and a single pass over the data can be implemented,
people have paid amount of attention to a class of methods under these condition, called streaming/online PCA.
The most natural and simple method was designed by Oja and his coauthor \cite{oja1982simplified,ojaK1985stochastic}:
first choose an initial guess $U^{(0)}\in\mathbb{R}^{d\times p}$ with $(U^{(0)})^{\T}U^{(0)}=I$,
and then iteratively update 
\[
	U^{(n)}=\Pi\left([I_d+\eta_nX^{(n)}(X^{(n)})^{\T}]U^{(n-1)}\right)=[I_d+\eta_nX^{(n)}(X^{(n)})^{\T}]U^{(n-1)}S^{(n)},
\]
where $\Pi(A)$ is an orthonormal projector such that $\Pi(A)^{\T}\Pi(A)=I_p$ and $\subspan(A)=\subspan(\Pi(A))$,
and $S^{(n)}$ is used to denote the normalization matrix.
%Another method is proposed by Krasulina \cite{krasulina1969method}
There are three classes of hyperparameters:
\begin{enumerate}
	\item the initial guess $U^{(0)}$: usually first generate $\wtd U^{(0)}$ of which each entry follows the standard Gaussian distribution $N(0,1)$, and then obtain $U^{(0)}$ by QR decomposition.
		Note that in this setup, $U^{(0)}$ is uniformly sampled from all the $p$-dimensional subspaces under the Haar invariant probability measure (see eg. \cite{muirhead1982aspects}).
		\item the learning rates $\eta_n$: there are different strategies to choose them.
		Two common setups are the constant learning rates $\eta_n=\eta_o$, and harmonic learning rates $\eta_n\propto \frac{1}{n}$.
		\item the normalization matrices $S^{(n)}$: Two common ways to obtain the orthonormal basis are QR decomposition, and polar decomposition \cite{abedmeraimACH2000orthogonal,liangGLL2017nearly:arxiv}.
\end{enumerate}

Although the Oja's algorithm was developed nearly 40 years ago and it works well in practice, its convergence behavior is limited until recently.
Most theoretical results come out since 2014.
As was argued by Allen-Zhu and Li~\cite{allenzhuL2017first},
the convergence rate of the Oja's method has several features:
\begin{enumerate}
	\item efficient: the rate only depends on the dimension $d$ logarithmically.  In fact, the dependence on $d$ can be removed.
	\item gap-free: the rate is independent of the eigenvalue gap 
		\[
			\gamma=\lambda_p-\lambda_{p+1}.
		\]
		In details, the feature tells that $\N{\sin\Theta(\subspan(U^{(n)}),\subspan(u_1,\dots,u_{q}))}$ is bounded by a factor $\wtd\gamma^{-2}$ rather than $\gamma^{-2}$, where $\wtd\gamma$ is an arbitrary chosen threshold, and $\lambda_p-\lambda_{q}<\wtd\gamma\le\lambda_p-\lambda_{q+1}$.
		\item global: the algorithm is allowed to start from a random initial guess.
\end{enumerate}

Some recent works \cite{shamir2016convergence}
studied the convergence of the online PCA for the most significant principal component, i.e., $u_1$, from different points of view
and obtained some results for the case where the samples are almost surely uniformly bounded.
De Sa et al.~\cite{desaRK2015global} studied a different
but closely related problem, in which the angular part is equivalent to the online PCA,
and obtained some convergence results.
Li et al.~\cite{liWLZ2017near} analyzed for the distributions with sub-Gaussian tails, and for this case the samples of this kind of distributions may be unbounded.
For more details of comparison, the reader is referred to \cite{liWLZ2017near}. %\cite{2016arXiv160305305L}.

For the subspace online PCA,
some recent works studied the convergence for the case where the samples are almost surely uniformly bounded.
In a series of papers~\cite{aroraCLS2012stochastic,aroraCS2013stochastic,marinovMA2018streaming,mianjyA2018stochastic},
Arora et al.\ studied \cref{eq:PCA-opt} and its variations via direct optimization approaches, namely using convex relaxation and adding regularizations.
The Oja's algorithm falls into one variant of their methods.
Hardt and Price~\cite{hardtP2014noisy} and Balcan et al.~\cite{balcanDWY2016improved} treated the method as a noisy power method and analyzed its convergence.
Shamir~\cite{shamir2016fast} first proved the convergence is efficient with a good initial guess.
%Li et al.~\cite{liLL2016rivalry} investigated the convergence.
Garber et al.~\cite{garberHJKMNS2016faster} used the shift-and-invert technique to speed up the convergence
but their analysis was only done for the top eigenvector.
Allen-Zhu and Li~\cite{allenzhuL2017first} analyzed the method and proposed  a faster variant of subspace online PCA iteration,
and firstly showed the gap-free feature of the convergence and also gave a lower bound for the gap-free feature.
Very recently Huang et al.~\cite{huang2021streaming} analyzed the problem using the new matrix concentration inequalities and proved stronger upper bounds.
Liang et al.~\cite{liangGLL2017nearly:arxiv} went further along the way of \cite{liWLZ2017near} and gave an convergence analysis for sub-Gaussian distributions.

The convergence rates obtained in some previous works and this paper are presented in \Cref{tab:comparison-of-some-results}.
\begin{table}[tp]
	\centering
	\begin{tabular}{lccccc}
		\hline
		Paper                                        & Global convergence                                                                      & Local convergence                                                                      & Unbounded & Block & Gap-free \\
		\hline
		De Sa et al.~\cite{desaRK2015global}         & $\dfrac{\lambda_{1\sim d}^2d}{\gamma^2n}\ln \dfrac{d}{\delta}$                          & $\dfrac{\lambda_{1\sim d}^2d}{\gamma^2n}\ln \dfrac{d}{\delta}$                         & No        & No    & No       \\
		Hardt and Price~\cite{hardtP2014noisy}       & $\dfrac{\lambda_{1\sim d}^2\lambda_pd}{\gamma^3n}\ln \dfrac{nd}{\delta}$                & $\dfrac{\lambda_{1\sim d}^2\lambda_pd}{\gamma^3n}\ln \dfrac{nd}{\delta}$               & No        & Yes   & No       \\
		%\cite{2015arXiv150103796B} &              &             &           &       &           &                        \\
		Shamir~\cite{shamir2016convergence}          & $\dfrac{\lambda_{1\sim d}^2d}{\gamma^2n}\dfrac{(\ln n)^2}{(1-\delta)^2}$                & $\dfrac{\lambda_{1\sim d}^2}{\gamma^2n}\dfrac{(\ln n)^2}{(1-\delta)^2}$                & No        & No    & No       \\
		Shamir~\cite{shamir2016fast}                 & ---                                                                                     & $\dfrac{\lambda_{1\sim d}^2}{\gamma^2n}\ln \dfrac{n}{\delta}$                          & No        & Yes   & No       \\
		Balcan et al.~\cite{balcanDWY2016improved}   & $\dfrac{\lambda_{1\sim p}^2\lambda_pd}{\gamma^3n}\ln \dfrac{nd}{\delta}$                & $\dfrac{\lambda_{1\sim p}^2\lambda_pd}{\gamma^3n}\ln \dfrac{nd}{\delta}$               & No        & Yes   & No       \\
		Jain et al.~\cite{jainJKNS2016streaming}     & $\dfrac{M_4}{\gamma^2n}\ln \dfrac{d}{\delta}$                    & $\dfrac{M_4}{\gamma^2n}\ln \dfrac{d}{\delta}$                   & No        & No    & No       \\
		Li et al.~\cite{liWLZ2017near}               & $\dfrac{\lambda_{1\sim p}\lambda_{p+1\sim d}}{\gamma^2n}\dfrac{\ln n}{1-\delta}$        & $\dfrac{\lambda_{1\sim p}\lambda_{p+1\sim d}}{\gamma^2n}\dfrac{\ln n}{1-\delta}$       & Yes       & No    & No       \\
		Allen-Zhu and Li~\cite{allenzhuL2017first}   & $\dfrac{\lambda_{1\sim p}\lambda_{1\sim d}}{\gamma^2(n-n_o)}\ln\dfrac{d}{\gamma\delta}$ & $\dfrac{\lambda_{1\sim p}\lambda_{1\sim d}}{\gamma^2n}\ln\dfrac{d}{\gamma\delta}$      & No        & Yes   & Yes      \\
		Liang et al.~\cite{liangGLL2017nearly:arxiv} & $\dfrac{\lambda_{1\sim p}\lambda_{p+1\sim d}}{\gamma^2n}\dfrac{\ln n}{1-\delta^{p^2}}$  & $\dfrac{\lambda_{1\sim p}\lambda_{p+1\sim d}}{\gamma^2n}\dfrac{\ln n}{1-\delta^{p^2}}$ & Yes       & Yes   & No       \\
		Huang et al.~\cite{huang2021streaming}       & $\dfrac{M_4}{\gamma^2(n-n_o)}\ln\dfrac{p}{\gamma\delta}$       & $\dfrac{M_4}{\gamma^2n}\ln\dfrac{p}{\gamma\delta}$            & No        & Yes   & No       \\
		\hline
		This paper 
                                                     & $\dfrac{\lambda_{1\sim p}\lambda_{p+1\sim d}(n-n_o)}{\gamma^2n^2(1-\delta)}$            & $\dfrac{\lambda_{1\sim p}\lambda_{p+1\sim d}}{\gamma^2n(1-\delta)}$                    & Yes       & Yes   & Yes      \\
		\hline
	\end{tabular}
	\caption{Comparison of some results}
	\label{tab:comparison-of-some-results}
		\begin{itemize}
			\item The term $M_4$ represents any quantity related to $\E{\N*{\bX\bX^{\T}-\Sigma}}$, or the fourth central moment (not necessarily the same in different results).
			\item
				In some results the term $n_o$ appears in the global convergence, and it represents the number of samples needed in the so-called ``Phase I'' or ``Cold Start'' process.
			\item 
				Note that there are two types of the dependency on $\delta$ in \Cref{tab:comparison-of-some-results}:
				one is $\ln\frac{1}{\delta}$, which goes to infinity as $\delta\to 0$; the other is $\frac{1}{(1-\delta^*)^*}$, which goes to $0$ as $\delta\to 0$.
				Clearly the latter term can be replaced by an absolute constant, or equivalently, the rate does not explicitly rely on $\delta$ (but implicitly, for $n\ge \OO\left((\ln \delta)^{-4}\right)$ as is shown in \Cref{lm:-cite-theorem-2-4-huang2021streaming,thm:main}).

			\item 
				In some results, such as Jain et al.~\cite[Theorem~1.3]{jainJKNS2016streaming}, there is no $\ln d$ or $\ln \frac{1}{\delta}$ factor, which seems to contradict with what we list in \Cref{tab:comparison-of-some-results} (based on their Theorem~1.2 or~4.1 actually).
				However, the assumption there is much stronger: first their success probability is $3/4$, rather than $1-\delta$, which removes the dependency on $\delta$; then they need $n\ge \OO(d^{1/10})$ which is much larger than $\OO\left((\ln d)^4\right)$ here.
		\end{itemize}
\end{table}
%Some quantities appearing in the table is given below.
The sum of some consecutive eigenvalues is written as
\[
	\lambda_{i_1\sim i_2}:=\lambda_{i_1}+\dots+\lambda_{i_2},\quad
	1\le i_1\le i_2\le d.
\]
The listed convergence rates are read as:
with probability $1-\delta$, using $n$ samples, or equivalently after $n$ iterations, the Oja's algorithm produces an approximation $\subspan(U^{(n)})$ of the principal subspace satisfying $\N{\sin\Theta(\subspan(U^{(n)}),{\mathcal{U}}_*)}\le (\text{the rate})$. 
The global convergence rate is given for the case that the initial guess is random generated,
while the local convergence rate is given for the case that the initial guess satisfies $ \N{\tan\Theta(\subspan(U^{(0)}),{\mathcal{U}}_*)}$ is bounded by an absolute constant like $1$.
%More discussion on \Cref{tab:comparison-of-some-results} is given in \cref{ssec:more-discussion-on-the-comparison}.

The convergence rates listed in \Cref{tab:comparison-of-some-results}, except ours, include a poly-logarithmic factor, which leads people to say the Oja's method is \emph{nearly optimal}.
However, in this paper we will show the poly-logarithmic factor can be removed. 
In other words,
the convergence rate, namely the upper bound of the approximation, exactly matches the lower bound \cref{eq:minimax-bound} of an approximation obtained by the offline/classical PCA up to a constant factor.
Hence in some sense, we may say the Oja's method is \emph{optimal}.
To the best of our knowledge, it is the first time to point out this feature of the online method.
%This is the main purpose of this paper. 

Other results we will show in this paper include:
\begin{itemize}
	\item the strategy of choosing the normalization matrices does not matter much on the convergence rate.
		Hence we may choose a strategy that has advantages on computation or practical consideration.
	\item the iteration process is somehow decoupled, and thus the gap-dependent and gap-free considerations can be treated in the same framework. 
		This would shed light on the convergence analysis of other online algorithms.
	\item a lower bound for sub-Gaussian distributions on the gap-free feature is given, which ensures that the Oja's method is optimal.
\end{itemize}

The rest of this paper is organized as follows.
In \Cref{sec:preliminaries} we make preparations for discussing the convergence analysis of the Oja's method. %including reviewing the definitions of the canonical angles between two subspaces (that may have different dimensions) and presenting two related lemmas to be used in the later estimations, reviewing the definition of Orlicz norms and sub-Gaussian distributions, and simply simplifying the iteration process produced by the Oja's method.
The main results, namely the convergence analysis, are stated in \Cref{sec:main-results},
while their proofs are provided in \Cref{sec:proofs} due to the complexities of the contained heavy calculations.
\Cref{sec:conclusion} gives some concluding remarks.

%\subsection{More Discussion on the Comparison}\label{ssec:more-discussion-on-the-comparison}

\subsection{Notation}\label{ssec:notation}
%\emph{Notations.} %\label{para:notations-}
$I_n$ (or simply $I$ if its dimension is
clear from the context) is the $n\times n$ identity matrix and $e_j$ is its $j$th column (usually with dimension determined by the context).
For a matrix $X$, $\sigma(X)$, $\N{X}_2$ and $\N{X}_{\F}$ are the multiset of the singular values,
the spectral norm, and the Frobenius norm of $X$, respectively.
For two matrices or vectors $X,Y$, $X\circ Y$ is the Hadamard/entrywise product of $X$ and $Y$ of the same size.

For any matrix $X$, $X_{(i,j)}$ is the $(i,j)$th entry of $X$, and 
$X_{(i:j,:)}$ is the submatrix of $X$ consisting of its row $i$ to row $j$.
For any vector or matrix $X,Y$,
$X\le Y$ means $X_{(i,j)}\le Y_{(i,j)}$ for any $i,j$. $X$``$\ge,>,<$''$Y$ can be similarly understood.
%$X\le \alpha$ ($X<\alpha$) for a scalar $\alpha$ means $X_{(i,j)}\le \alpha$ ($X_{(i,j)}< \alpha$) for any $i,j$;
%similarly $X\ge\alpha$ and $X>\alpha$.

For a subset or an event $\eventA$, $\eventA^{\rmc}$ is the complement set of $\eventA$.
By $\sigma\set{\eventA_1,\dots, \eventA_p}$ we denote the $\sigma$-algebra generated by the events $\eventA_1,\dots, \eventA_p$.
$\E{\bX\overevent\eventA}:=\E{\bX\ind{\eventA}}$ denotes the expectation of a random variable $\bX$ over event $\eventA$.
Note that
$\E{\bX\overevent\eventA}=\E{\bX\given \eventA}\prob{\eventA}$.
%\begin{equation}\label{eq:expectation-over2cond}
%\end{equation}
For a random vector or matrix $\bX$, $\E{\bX}:=\left[ \E{\bX_{(i,j)}} \right]$.
Note that $\N{\E{\bX}}_{\UI}\le\E{\N{\bX}_{\UI}}$ for $\UI=2,{\scs\F}$.
Write $\covc{\bX,\bY}:=\E{[\bX-\E{\bX}]\circ[\bY-\E{\bY}]}$ and $\varc{\bX}:=\covc{\bX,\bX}$.

For any scalar $x,y$, $x\vee y=\max\set{x,y}$, $x\wedge y=\min\set{x,y}$.

\section{Preliminaries}\label{sec:preliminaries}
\subsection{Canonical Angles between Two Subspaces}\label{ssec:canonical-angles-between-two-subspaces}
We are interested in the distance of two linear subspaces.
So we introduce the canonical angles between them in order to give quantities to represent their distance.
\begin{definition}[{\cite{bjorkG1973numerical}}]\label{dfn:canonical-angles}
Given two subspaces $\mathcal{X},\,\mathcal{Y}\subseteq\mathbb{R}^d$ with $\dim \mathcal{X}=p\le \dim \mathcal{Y}=q$.
The \emph{principal/canonical angles} $\theta_j\in[0,\pi/2]$ between $\mathcal{X}$ and $\mathcal{Y}$ are recursively defined for $j=1,\dots,p$ by

		\[
		\begin{aligned}
			\cos \theta_j=\sigma_j= \max_{u\in \mathcal{X},v\in \mathcal{Y}} u^{\T}v=u_j^{\T}v_j
			\quad\text{subject to}\;&
			\N{u}_2=\N{v}_2=1,
			\\ &u_i^{\T}u=0, v_i^{\T}v=0,i=1,\dots,j-1.
		\end{aligned}
	\]
\end{definition}
It can be verified that 
$\sigma_1\ge\cdots\ge\sigma_p$ are the singular values of $X^{\T}Y$, where $X,Y$ are orthonormal basis matrices of $\mathcal{X},\mathcal{Y}$ respectively.
The angles are in non-decreasing order: $\theta_1\le\dots\le\theta_p$.
Moreover, it can be seen that $\sigma_j$ or $\theta_j$ are independent of the basis matrices, which are not unique.

Write
\begin{equation*}\label{eq:mat-angles-XY}
\Theta(\mathcal{X},\mathcal{Y})=\diag(\theta_1,\ldots,\theta_p).
\end{equation*}
Here we add ``$(\mathcal{X},\mathcal{Y})$'' to emphasize the quantity is defined for two subspaces $\mathcal{X},\mathcal{Y}$.
In particular, if $p=q$,
$\N{\sin\Theta(\mathcal{X},{\cal Y})}_{\UI}$ for $\UI=2,{\scriptstyle \F}$
are metrics on the set consisting for all $p$-dimensional subspaces of $\mathbb{R}^d$ \cite[Section~II.4]{stewartS1990matrix}.

For matrices $X,Y$, $\Theta(X,Y):=\Theta(\subspan(X),\subspan(Y))$.

In what follows, we give a quantity easy to compute to estimate the distance between one subspace and a particular subspace.

Given $p\le q$, for any matrix $X\in \mathbb{R}^{d\times p}$ with nonsingular $X_{(1:p,:)}$, write 
\begin{equation*}\label{eq:scrT-dfn}
\scrT_{p,q}(X):=X_{(p+1:q,:)}X_{(1:p,:)}^{-1},\quad
	\scrT_q(X):=X_{(q+1:d,:)}X_{(1:p,:)}^{-1},
\end{equation*}
which are submatrices of $\scrT(X):= \scrT_p(X)$.

\begin{lemma}\label{lm:Tan(Theta)}
We have for $\UI=2,{\scs \F}$
\begin{equation}\label{eq:Tan(Theta)}
\N*{\tan\Theta(X,\begin{bmatrix} I_q \\ 0 \end{bmatrix})}_{\UI}
  \le\N{\scrT_q(X)}_{\UI}.
\end{equation}
In particular, if $p=q$, then the inequality ``$\le$'' can be replaced by ``$=$''.
\end{lemma}
\begin{proof}
%Let $Y=\begin{bmatrix} I_q \\ 0 \end{bmatrix}\in \mathbb{R}^{d\times q}$.
For the readability, we use $\scrT_{p,q},\scrT_p,\scrT_q$ only and drop ``$(X)$''.
Then $\cos\theta_j$ for $j=1,\dots,p$ are the singular values of
\[
	\left[I_p+\scrT_p^{\T}\scrT_p\right]^{-1/2}\begin{bmatrix} I_p \\ \scrT_{p,q} \\ \scrT_q \end{bmatrix}^{\T}\begin{bmatrix} I_q \\ 0 \end{bmatrix}
	=\left[I_p+\scrT_p^{\T}\scrT_p\right]^{-1/2}\begin{bmatrix} I_p & \scrT_{p,q} \end{bmatrix}
	.
\]
Thus, $\tan^2\theta_j=\sec^2\theta_j-1$ for $j=1,\dots,p$ are the eigenvalues of
\begin{align*}
	&\quad\left(\left[I_p+\scrT_p^{\T}\scrT_p\right]^{-1/2}\begin{bmatrix} I_p & \scrT_{p,q} \end{bmatrix}\begin{bmatrix} I_p & \scrT_{p,q} \end{bmatrix}^{\T}\left[I_p+\scrT_p^{\T}\scrT_p\right]^{-1/2}\right)^{-1}-I
	\\&=
	\left[I_p+\scrT_p^{\T}\scrT_p\right]^{1/2}\left[ I_p + \scrT_{p,q}^{\T}\scrT_{p,q}\right]^{-1}\left[I_p+\scrT_p^{\T}\scrT_p\right]^{1/2}-I
	\\&=
	\left[I_p+\scrT_p^{\T}\scrT_p\right]^{1/2}\left(\left[ I_p + \scrT_{p,q}^{\T}\scrT_{p,q}\right]^{-1}\left[I_p+\scrT_p^{\T}\scrT_p\right]-I\right)\left[I_p+\scrT_p^{\T}\scrT_p\right]^{-1/2}
	\\&=
	\left[I_p+\scrT_p^{\T}\scrT_p\right]^{1/2}\left[ I_p + \scrT_{p,q}^{\T}\scrT_{p,q}\right]^{-1}\scrT_q^{\T}\scrT_q\left[I_p+\scrT_p^{\T}\scrT_p\right]^{-1/2}
	,
\end{align*}
and also the eigenvalues of 
\[
	\left[ I_p + \scrT_{p,q}^{\T}\scrT_{p,q}\right]^{-1/2}\scrT_q^{\T}\scrT_q\left[I_p+ \scrT_{p,q}^{\T}\scrT_{p,q}\right]^{-1/2}.
\]
Let $\tau_1\ge\dots\ge\tau_p$ be the eigenvalues of $\scrT_q^{\T}\scrT_q$.
By the Ostrowski theorem \cite[Theorem~4.5.9]{hornJ1985matrix},
\[
	\tan^2\theta_j\le \tau_j,
\]
which implies \cref{eq:Tan(Theta)}.
\end{proof}
Note that \cite[Lemma~2.1]{liangGLL2017nearly:arxiv} is a special case of \Cref{lm:Tan(Theta)}.

\subsection{Orlicz Norms}\label{ssec:orlicz-norms}

We are concerned with random variables/vectors that have  a sub-Gaussian distribution.
To that end, we first introduce the Orlicz $\psi_\alpha$-norm of a random variable/vector.
More details can be found in \cite{vaartW1996weak}.

\begin{definition}\label{dfn:orlicz-norm}
	The \emph{Orlicz $\psi_\alpha$-norm\/} of a random variable $\bX\in \mathbb{R}$ is defined as
	\[
		\N{\bX}_{\psi_\alpha}:=\inf\set*{\xi>0 :  \E{\exp\left( \abs*{\frac{\bX}{\xi}}^\alpha \right)}\le2},
	\]
	and the \emph{Orlicz $\psi_\alpha$-norm\/} of a random vector $\bX\in \mathbb{R}^d$ is defined as
	\[
		\N{\bX}_{\psi_\alpha}:=\sup_{\N{v}_2=1}\N{v^{\T}\bX}_{\psi_\alpha}.
	\]
	We say that random variable/vector $\bX$ follows a \emph{sub-Gaussian distribution\/} if $\N{\bX}_{\psi_2}<\infty$.
\end{definition}
By the definition, we conclude that any bounded random variable/vector follows a sub-Gaussian distribution.

The basic properties of sub-Gaussian distributions are listed in \Cref{lm:-cite-lemma-5-5-and-5-10-5-12-vershynin2012introduction-}.
\begin{lemma}[{\cite[(5.10)--(5.12)]{vershynin2012introduction}}]\label{lm:-cite-lemma-5-5-and-5-10-5-12-vershynin2012introduction-}
	Every sub-Gaussian random variable $\bX\in \mathbb{R}$ with $\N{\bX}_{\psi_2}=\psi$ satisfies:
	\begin{enumerate}
		\item \label{itm:vershynin:5.10} $\prob{\abs{\bX}>t}\le \exp(1-c\psi^{-2}t^2)$ for $t\ge 0$;
		\item \label{itm:vershynin:5.11} $\E{\abs{\bX}^p}\le\psi^p p^{p/2}$ for $p\ge 1$;
		\item \label{itm:vershynin:5.12} if $\E{\bX}=0$, then $\E{\exp(t\bX)}\le \exp(C\psi^2t^2)$ for $t\in \mathbb{R}$,
	\end{enumerate}
	where $C>0,c>0$ are absolute constants.
\end{lemma}

Moreover, if $\N{\bX}_{\psi_1}<\infty$, then $\bX$ follows a \emph{sub-exponential} distribution.
Our analysis below can be easily generalized to sub-exponential random vectors, and will not be discussed.

\subsection{Detailed Algorithm and Assumptions}\label{ssec:detailed-algorithm}
Here we write down the detailed algorithm in \Cref{alg:onlinePCA:ppr1}.
\begin{algorithm}[ht]
	\caption{Oja's Algorithm for Online PCA}\label{alg:onlinePCA:ppr1}
	\begin{algorithmic}[1]
		\STATE Choose $U^{(0)}\in\mathbb{R}^{d\times p}$ with $(U^{(0)})^{\T}U^{(0)}=I$,
		and use a regime to choose the learning rate $\eta_n=\rho_n\eta>0$.
		\FOR{$n=1,2,\dots$ until convergence}
		\STATE Take an $\bX$'s sample $X^{(n)}$;
		\STATE $Z^{(n)}= (U^{(n-1)})^{\T}X^{(n)}$;
		\STATE $\wtd U^{(n)}= U^{(n-1)}+\eta_n X^{(n)}(Z^{(n)})^{\T}$;
		%\STATE Find a proper $S^{(n)}$ to make $U^{(n)}=\wtd U^{(n)}S^{(n)}$ column orthonormal.
		\STATE Find an orthonormal basis of the subspace spanned by $\wtd U^{(n)}$, namely compute a column orthonormal matrix $U^{(n)}=\wtd U^{(n)}S^{(n)}$.
		\ENDFOR
	\end{algorithmic}
\end{algorithm}

The learning rate of the $n$-th iteration is $\eta_n=\rho_n\eta_o$.
Without loss of generality, we may assume $0<\rho_n<1$.
%Usually $\rho_i$ is a nonincreasing sequence, and then $\rho_1=1$.

The decomposition can be chosen as QR decomposition \cite{ojaK1985stochastic} or polar decomposition \cite{abedmeraimACH2000orthogonal}, or any other decomposition easy to compute.
However,
$S^{(n)}$ is always nonsingular. In fact,
noticing that $\wtd U^{(n)} =[I+\eta_n X^{(n)}(X^{(n)})^{\T}]U^{(n-1)}$,
since $I+\eta_{n+1} X^{(n)}(X^{(n)})^{\T}$ is positive definite and thus nonsingular, and $U^{(n)}$ is column orthonormal,
we know $\wtd U^{(n)}$ has full column rank,
which implies the fact.

Any statement we will make holds \emph{almost surely}.

To prepare our convergence analysis, we make a few assumptions.

\begin{assumption} \label{asm:simplify:X}
$\bX=[\bX_1, \bX_2, \ldots, \bX_d]^{\T}\in\mathbb{R}^d$ is a random vector.
\begin{enumerate}[label=(A-\arabic*),ref=A-\arabic*]
	\item \label{asm:simplify:X:moment}
		$\E{\bX}=0$, and $\Sigma:=\E{\bX\bX^{\T}}$ has the spectral decomposition \cref{eq:eigD-convar} satisfying
		$\lambda_p>\lambda_{p+1}$;
	\item \label{asm:simplify:X:subGaussian}
		$\psi:=\N{\Sigma^{-1/2}\bX}_{\psi_2}<\infty$.
	%\item \label{asm:simplify:X:trace}
		%$\lambda_{1\sim d}:=\trace(\Sigma)=\lambda_1+\dots+\lambda_d=1$.
\end{enumerate}
\end{assumption}

The principal subspace $\mathcal{U}_*$ is uniquely determined under \cref{asm:simplify:X:moment} of \Cref{asm:simplify:X}.
On the other hand, \cref{asm:simplify:X:subGaussian} of \Cref{asm:simplify:X}  ensures that all 1-dimensional marginals of $\bX$ have
sub-Gaussian tails, or equivalently,
$\bX$ follows a sub-Gaussian distribution.
%This is also an assumption that is used in \cite{liWLZ2017near}.
%\Cref{asm:simplify:X:trace} of \Cref{asm:simplify:X} is a technical assumption.

Using the substitutions 
\[
	\bX\leftarrow \lambda_{1\sim d}^{-1/2}\bX,
	X^{(n)}\leftarrow\lambda_{1\sim d}^{-1/2}X^{(n)},
	Z^{(n)}\leftarrow\lambda_{1\sim d}^{-1/2}Z^{(n)},
	\eta_n\leftarrow\lambda_{1\sim d}\eta_n,
\]
the iterations produced by \Cref{alg:onlinePCA:ppr1} and the rest terms $\wtd U^{(n)},U^{(n)}$ keep the same.
Hence any convergence result has to keep this homogeneous property.

Next we make a simplification on the problem.

%Without loss of generality, we may assume that the covariance matrix $\Sigma$ diagonal. Otherwise, we
%can perform a (constant) orthogonal transformation as follows.
Recall the spectral decomposition $\Sigma=U\Lambda U^{\T}$.
Instead of
the random vector $\bX$,  we  equivalently consider
\[
\bY\equiv [\bY_1,\bY_2,\ldots,\bY_n]^{\T}:=U^{\T}\bX.
\]
Accordingly, perform the same orthogonal transformation on all involved
quantities:
\begin{equation*}\label{eq:YV-dfn}
Y^{(n)}=U^{\T}X^{(n)}, \quad V^{(n)}=U^{\T}U^{(n)}, \quad
V_*=U^{\T}U_*=\begin{bmatrix} I_p \\ 0\end{bmatrix}.
\end{equation*}
%As consequences, we will have the equivalent versions of \Cref{alg:onlinePCA:ppr1}, \Cref{thm:thm1:ppr1,,thm:thm2:ppr1,thm:thm3:ppr1}.
Firstly, because
\[
(V^{(n-1)})^{\T}Y^{(n)}=(U^{(n-1)})^{\T}X^{(n)}, \quad (Y^{(n)})^{\T}Y^{(n)}=(X^{(n)})^{\T}X^{(n)},
\]
the equivalent version of \Cref{alg:onlinePCA:ppr1} is obtained by symbolically
replacing
all letters $X,\, U$ by
$Y,\,V$ while keeping their respective superscripts.
If the algorithm converges, it is expected that $\subspan(V^{(n)})\to\subspan(V_*)$.
Secondly, noting
\[
\N{\Sigma^{-1/2}\bX}_{\psi_2} = \N{U\Lambda^{-1/2}U^{\T}\bX}_{\psi_2} = \N{\Lambda^{-1/2}\bY}_{\psi_2},
\]
we can restate \Cref{asm:simplify:X} equivalently as
\begin{enumerate}[label=(A-\arabic*$'$),ref=A-\arabic*$'$]
	\item $\E{\bY}=0,\E{\bY\bY^{\T}}=\Lambda=\diag(\lambda_1,\dots,\lambda_d)$ with $\lambda_1\ge\dots\ge\lambda_p>\lambda_{p+1}\ge\dots\ge\lambda_d$;
\item $\psi:=\N{\Lambda^{-1/2}\bY}_{\psi_2}<\infty$.
%\item $\lambda_{1\sim d}:=\trace(\Sigma)=\lambda_1+\dots+\lambda_d=1$.
\end{enumerate}
Thirdly, all canonical angles between two subspaces are invariant under the orthogonal transformation.
Therefore the results given below
holds for not only $\bY$ but also $\bX$.
%for $\bY$ can be simply obtained by replacing all letters $X,\, U$ by
%$Y,\,V$ while keeping their respective superscripts.

If the algorithm converges, it is expected that
\[
	U^{(n)}\to U_*:= U\begin{bmatrix} I_p \\ 0\end{bmatrix}=[u_1,u_2,\ldots,u_p]
	\qquad\Leftrightarrow\qquad
	V^{(n)}\to V_*=\begin{bmatrix} I_p \\ 0\end{bmatrix}
\]
in the sense that 
\[
	\N{\sin\Theta(U^{(n)}, U_*)}_{\UI}\to 0
	\qquad\Leftrightarrow\qquad
	\N{\sin\Theta(V^{(n)}, V_*)}_{\UI}\to 0
\]
as $n\to\infty$.

By \Cref{lm:Tan(Theta)}, it is sufficient enough to prove $\N{\scrT(V^{(n)})}_{\UI}\to 0$.
Our results are based on this point.

To simplify the notations in our proofs, we introduce new notations for two particular submatrices of any vector $Y\in \mathbb{R}^{d}$, tall matrix $V\in\mathbb{R}^{d\times p}$ and diagonal matrix $\Lambda\in \mathbb{R}^{d\times d}$:
\begin{equation*}\label{eq:ol-ul}
	%Y=\kbordermatrix{ &\scs 1\\
		   %\scs p   & \ol Y \\
		   %\scs d-p & \ul Y},\quad
	%V=\kbordermatrix{ &\scs p\\
		   %\scs p   & \ol V \\
		   %\scs d-p & \ul V},\quad
		   %\Lambda=\kbordermatrix{ &\scs p &\scs d-p\\
		%\scs p   & \ol \Lambda & \\
	%\scs d-p & & \ul \Lambda},
	Y=
	\begin{blockarray}{cc}
		\scs 1 & \\
		\begin{block}{[>{\,}c<{\,}\Right{]}{,\;}>{\scs }c}
			\ol Y & p \\
			\ul Y & d-p \\
		\end{block}
	\end{blockarray}
	\qquad
	V=
	\begin{blockarray}{cc}
		\scs p & \\
		\begin{block}{[>{\,}c<{\,}\Right{]}{,\;}>{\scs }c}
			\ol V & p \\
			\ul V & d-p \\
		\end{block}
	\end{blockarray}
	\qquad
	\Lambda =
	\begin{blockarray}{ccc}
		\scs p & \scs d-p \\
		\begin{block}{[>{\,}cc<{\,}\Right{]}{.\;}>{\scs }c}
			\ol \Lambda &  & p \\
			& \ul \Lambda & d-p \\
		\end{block}
	\end{blockarray}
\end{equation*}
or equivalently 
\[
\ol Y=Y_{(1:p,:)}, \quad \ul Y=Y_{(p+1:d,:)},\qquad
\ol V=V_{(1:p,:)}, \quad \ul V=V_{(p+1:d,:)},
\]
and
\[
	\ol\Lambda=\diag(\lambda_1,\ldots,\lambda_p), \quad\ul\Lambda=\diag(\lambda_{p+1},\ldots,\lambda_d).
\]
\section{Main Results}\label{sec:main-results}
In what follows, we will state our main results and leave their proofs to another section because of their high complexity.
The main technique to prove the results is the same as Li et al.~\cite{liWLZ2017near} and Liang et al.~\cite{liangGLL2017nearly:arxiv}. 
The differences between the results are referred to:
\begin{itemize}
	\item the estimations are much sharper here;
	\item the learning rates are changing here, rather than a fixed learning rate in \cite{liWLZ2017near,liangGLL2017nearly:arxiv}.
\end{itemize}

First we introduce some quantities.

For $\kappa\ge0$, define
$\sphere(\kappa ):=\set{V\in\mathbb{R}^{d\times p} :  \sigma(\ol V)\subset [\frac{1}{\sqrt{1+\kappa^2}},1]}$.
It can be verified that
\begin{equation}\label{eq:sphereK-normT}
	V\in\sphere(\kappa )\Leftrightarrow \N{\scrT(V)}_2\le \kappa.
\end{equation}
For the sequence $V^{(n)}$, define
\[
	\Nout{\kappa }:=\min\set{n :  V^{(n)}\notin \sphere(\kappa )},\quad
	\Nin{\kappa }:=\min\set{n :  V^{(n)}\in \sphere(\kappa )}.
\]
$\Nout{\kappa }$ is the first step of the iterative process at which $V^{(n)}$ jumps from $\sphere(\kappa )$ to outside, and
$\Nin{\kappa }$ is the first step of the iterative process at which $V^{(n)}$ jumps from outside to $\sphere(\kappa )$.
For $\mu\ge 1$, define
\begin{equation*}\label{eq:Nqb-dfn}
	\Nqb{\mu }:= \max\set*{n\ge1 :  \N{Z^{(n)}}_2\le \lambda_{1\sim p}^{1/2}\mu^{1/2} ,\abs{Y^{(n)}_i}\le \lambda_i^{1/2}\mu^{1/2} ,i=1,\dots, d}+1.
\end{equation*}
$\Nqb{\mu }$ is the first step of the iterative process at which either $\abs{Y^{(n)}_i}>\lambda_i^{1/2}\mu^{1/2} $ for some $i$
or the norm of $Z^{(n)}$ exceeds $\lambda_{1\sim p}^{1/2}\mu^{1/2} $.
For $n<\Nqb{\mu }$, we have
$
	\N{Y^{(n)}}_2\le \lambda_{1\sim d}^{1/2}\mu^{1/2}
$
.

For convenience, we will set
$T^{(n)}=\scrT(V^{(n)})$,
and let $\fil_n=\sigma\set{Y^{(1)},\dots,Y^{(n)}}$ be the $\sigma$-algebra filtration, i.e., the information known by step $n$.

\subsection{Increments of One Iteration}\label{ssec:increments-in-one-iteration}
	In each iteration, 
\[
	\ol V^{(n+1)} = (\ol V^{(n)} + \eta_{n+1} \ol Y^{(n+1)} (Z^{(n+1)})^{\T})S^{(n)},
	\qquad
	\ul V^{(n+1)} = (\ul V^{(n)} + \eta_{n+1} \ul Y^{(n+1)} (Z^{(n+1)})^{\T})S^{(n)},
\]
	where $S^{(n)}$ is nonsingular as is stated above.
	According to the Sherman-Morrison formula%\cite[p.~95]{demm:1997}
	, we get $\ol V^{(n)} + \eta_{n+1}\ol Y^{(n+1)}(Z^{(n+1)})^{\T}$ or $\ol V^{(n+1)}$ is nonsingular, if and only if
	$1+\eta_{n+1} \xi_{n+1} \ol Y^{(n+1)}\ne 0$ where $\xi_{n+1}:=(Z^{(n+1)})^{\T}(\ol V^{(n)})^{-1}\ol Y^{(n+1)}$,
	and 
	\[
		(\ol V^{(n+1)})^{-1}=(S^{(n)})^{-1} \left(I-\frac{\eta_{n+1}}{1+\eta_{n+1}\xi_{n+1}}\ol V^{-1}\ol Y^{(n+1)}(Z^{(n+1)})^{\T}\right)\ol V^{-1}
		.
	\]
Hence 
\begin{align*}
	T^{(n+1)}
	&=\ul V^{(n+1)}(\ol V^{(n+1)})^{-1}
	\\&=
		(\ul V^{(n)} + \eta_{n+1} \ul Y^{(n+1)} (Z^{(n+1)})^{\T})S^{(n)}(S^{(n)})^{-1}\left(I-\frac{\eta_{n+1}}{1+\eta_{n+1}\xi_{n+1}}\ol V^{-1}\ol Y^{(n+1)}(Z^{(n+1)})^{\T}\right)\ol V^{-1}
		\\&	=
		(\ul V^{(n)} + \eta_{n+1} \ul Y^{(n+1)} (Z^{(n+1)})^{\T})\left(I-\frac{\eta_{n+1}}{1+\eta_{n+1}\xi_{n+1}}\ol V^{-1}\ol Y^{(n+1)}(Z^{(n+1)})^{\T}\right)\ol V^{-1}
		.
\end{align*}
Clearly the choice of $S^{(n)}$ does not matter on the convergence of $T^{(n)}$.
In other words, the strategy of choosing the normalization matrices does not matter much on the convergence rate.

In the following, we need to estimate $T^{(n+1)}-T^{(n)}$,
and the results are listed in \Cref{lm:diff-T}.
\begin{lemma}\label{lm:diff-T}
Suppose
\begin{equation}
	2\lambda_{1\sim p}\sqrt{\kappa^2+1}\mu\eta_{n+1}\le1.
	\label{eq:beta-mu-eta_p}
\end{equation}
Let
$\tau=\N{T^{(n)}}_2$.
If $n<{\Nqb{\mu }\wedge \Nout{\kappa}}$, then the following statements hold.
	\begin{enumerate}
		\item \label{itm:lm:diff-T:well-defined}
			$T^{(n)}$ and $T^{(n+1)}$ are well-defined.
		\item \label{itm:lm:diff-T:normdT}
			$\N{T^{(n+1)}-T^{(n)}}_2\le 2\mu\eta_{n+1}[\nu^{1/2}\lambda_{1\sim p}(1+\tau^2)+\nu_1\lambda_{1\sim p}\tau]$,
			where $\nu_1 = 1\vee\nu,\nu=\frac{\lambda_{p+1\sim d}}{\lambda_{1\sim p}}$.
		\item \label{itm:lm:diff-T:dT}
			Define $R_E^{(n)}$ by $\E{T^{(n+1)}-T^{(n)}\given \fil_n}=\eta_{n+1} (\ul\Lambda T^{(n)}-T^{(n)}\ol\Lambda )+ R_E^{(n)}$. Then
			\[
				\N{R_E^{(n)}}_2\le
				2\lambda_1\lambda_{1\sim p}\mu\eta_{n+1}^2\tau(1+\tau^2)^{1/2}.
			\]
		\item \label{itm:lm:diff-T:estimate}
			Let $H_{\circ}=\varc{\ul Y^{(n+1)}(\ol Y^{(n+1)})^{\T}}$ and define $R_{\circ}^{(n)}$ by $\varc{T^{(n+1)}-T^{(n)}\given\fil_n}=\eta_{n+1}^2H_{\circ}+R_{\circ}^{(n)}$. Then
			\begin{enumerate}
				\item \label{itm:lm:diff-T:varcdT}
					$H_{\circ}%=\varc{\ul Y\ol Y^{\T}}
					\le 16\psi^4 H$, where $H=[\eta_{ij}]_{(d-p)\times p}$ with $\eta_{ij}=\lambda_{p+i}\lambda_j$ for $i=1,\dots,d-p$, $j=1,\dots,p$;
				\item \label{itm:lm:diff-T:normR_H}
					$
					\begin{aligned}[t]
						\N{R_{\circ}^{(n)}}_2\le 
						2\nu_1\nu^{1/2}\lambda_{1\sim p}^2\mu^2\eta_{n+1}^2\tau \Big(1+\left[1+\nu_1\nu^{-1/2}\right]\tau +\tau ^2+\frac{1}{2}\tau ^3\Big)
						\\+ 8\nu\lambda_{1\sim p}^3\mu^3\eta_{n+1}^3(1+\tau^2)^{1/2} \left[1+\tau^2+\nu_1\nu^{-1/2}\tau\right]^2
						.
					\end{aligned}
					$
			\end{enumerate}
	\end{enumerate}
\end{lemma}

\subsection{Whole Iteration Process with a Good Initial Guess}\label{ssec:quasi-power-iteration-process}
Define $D^{(n+1)}=T^{(n+1)}-\E{T^{(n+1)}\given\fil_n}$.
It can be seen that
\begin{gather*}
T^{(n)}-\E{T^{(n)}\given\fil_n}=0,\quad
\E{D^{(n+1)}\given\fil_n}=0,\\
\E{D^{(n+1)}\circ D^{(n+1)}\given\fil_n}=\varc{T^{(n+1)}-T^{(n)}\given\fil_n}.
\end{gather*}
By \cref{itm:lm:diff-T:dT} of \Cref{lm:diff-T}, we have
\begin{align*}
	T^{(n+1)}
	&= D^{(n+1)} +T^{(n)}+\E{T^{(n+1)}-T^{(n)}\given\fil_n}
	\\ &= D^{(n+1)} +T^{(n)}+\eta_n (\ul\Lambda T^{(n)}-T^{(n)}\ol\Lambda )+ R_E^{(n)}
	\\ &= \opL_{n+1} T^{(n)}+D^{(n+1)} + R_E^{(n)},
\end{align*}
where $\opL_{n+1}\colon T\mapsto T+\eta_{n+1}\ul\Lambda T-\eta_{n+1} T\ol\Lambda $ is a bounded linear operator.
It can be verified that $\opL_{n+1} T=L_{n+1}\circ T$, the Hadamard product of $L_n$ and $T$, where
$L_{n+1}=[\lambda_{ij}^{(n+1)}]_{(d-p)\times p}$ with $\lambda_{ij}^{(n+1)}=1+\eta_{n+1} (\lambda_{p+i}-\lambda_j)$.
Clearly $\opL_{n_1}\opL_{n_2}=\opL_{n_2}\opL_{n_1}$ for any $n_1,n_2$.
Moreover, it can be shown that\footnote{
	Here we drop the superscript ``$\cdot^{(n+1)}$'' on $\lambda_{ij}$ and the subscript ``$\cdot_{n+1}$'' on $\eta,\opL$.
	Since  $\lambda(\opL)=\set{\lambda_{ij} :  i=1,\dots,d-p,\,j=1,\dots,p}$,
	the spectral radius $\rho(\opL)=1-\eta_{n+1}(\lambda_p-\lambda_{p+1})$.
	Thus for any $T$,
	\[
		\N{\opL T}_{\UI}=\N{T(I-\eta\ol\Lambda )+\eta\ul\Lambda T}_{\UI}
		\le\N{I-\eta\ol\Lambda }_2\N{T}_{\UI}+\N{\eta\ul\Lambda }_2\N{T}_{\UI}
		=(1-\eta\lambda_p+\eta\lambda_{p+1})\N{T}_{\UI}
		=\rho(\opL)\N{T}_{\UI}		,
	\]
	which means $\N{\opL}_{\UI}\le\rho(\opL)$. This ensures $\N{\opL}_{\UI}=\rho(\opL)$.
}
$\N{\opL_{n+1}}_{\UI}=\rho(\opL_{n+1})=1-\eta_{n+1}\gamma$, where $\N{\opL_{n+1}}_{\UI}=\sup_{\N{T}_{\UI}=1}\N{\opL_{n+1} T}_{\UI}$ is
an operator norm induced by the matrix norm $\N{\cdot}_{\UI}$.
Recursively,
\[
	T^{(n)}
	= \opL_n\dotsm\opL_1 T^{(0)}+D^{(n)}+\sum_{s=1}^{n-1}\opL_n\dotsm\opL_{s+1} D^{(s)} +R_E^{(n-1)}+\sum_{s=1}^{n-1} \opL_n\dotsm\opL_{s_1}R_E^{(s-1)}
	.
\]
Let us introduce some notation here.
Define $\dps\prod_{p=a}^b (\cdot)=1$ or $\opI$, the identical mapping, if $a>b$.
	Write 
\[
	F_*^{(n',n)}:=\prod_{r=n'}^n\N{\opL_r}_{\UI}=\prod_{r=n'}^n(1-\eta_r\gamma), 
	\qquad 
	F_{D,i,j}^{(n',n)}:=\sum_{s=n'}^{n}\eta_s^i\prod_{r=s+1}^{n}\N{\opL_r}_{\UI}^j =\sum_{s=n'}^{n}\eta_s^i\prod_{r=s+1}^{n} (1-\eta_r\gamma)^j.
\]
Suppose $F_{D,i,j}^{(1,n)}\le C_{D,i,j}\gamma^{-1}\eta_n^{i-1}$ for any $n$, where $C_{D,i,j}$ is an absolute constant,
which can be easily examined in any specific strategy to choose $\eta_n$.

For $s>0$ and $\eta_*\gamma<1$, define
\begin{equation*}\label{eq:Ns-dfn}
	N_s^{(n')}:=\min\set*{n\in \mathbb{N} :  F_*^{(n',n)}\le (\eta_*\gamma)^s}
		,
\end{equation*}
which implies $F_*^{(n',N_s^{(n')})}\le(\eta_*\gamma)^s<F_*^{(n',N_s^{(n')}-1)}$.
Define
\begin{equation*}\label{eq:I1I2I3}
	\begin{aligned}[b]
		T^{(n)}
		&= \opL_n\dotsm\opL_1 T^{(0)}+D^{(n)}+\sum_{s=1}^{n-1}\opL_n\dotsm\opL_{s+1} D^{(s)} +R_E^{(n-1)}+\sum_{s=1}^{n-1} \opL_n\dotsm\opL_{s_1}R_E^{(s-1)}
		\\&= \left(\prod_{r=1}^n\opL_r\right) T^{(0)}+\sum_{s=1}^{n}\left(\prod_{r=s+1}^{n}\opL_r\right)  D^{(s)} +\sum_{s=1}^{n}\left(\prod_{r=s+1}^{n}\opL_r\right) R_E^{(s-1)}
		\\&=: T_*^{(n)}+T_D^{(n)}+T_R^{(n)}.
	\end{aligned}
\end{equation*}

Define events 
\begin{gather*}
	\eventM_n(\kappa,\mu)=\set*{ \N{T^{(n)}-T_*^{(n)}}_2\le
	\frac{1}{2}\upsilon(1+\kappa^2)\mu^{3/2}\eta_n^{1/2}\gamma^{1/2}
}, \label{eq:event-cMn}\\
	\eventT_n(\kappa)=\set*{\N{T^{(n)}}_2\le  \kappa},\quad 
	\eventQ_n(\mu)=\set*{n<\Nqb{\mu }},   \label{eq:event-cTn}
\end{gather*}
where $\upsilon=
8C_{D,2,1}\lambda_1\lambda_{1\sim p}\gamma^{-2}
$.
		
It can be shown that under some conditions, if the initial $V^{(0)}$ is not too bad, then with high probability $\N{T^{(n)}}_2$ will never become too large and eventually become as small as expected.
The formal statement is given in \Cref{lm:prop5:ppr1}.
\begin{lemma}\label{lm:prop5:ppr1}
	Let $%N_0:=N_{-\ln\kappa/\ln(\eta_*\gamma)}^{(1)},
	N_1:=N_{(\ln\varepsilon -\ln\kappa)/\ln(\eta_*\gamma)}^{(1)}$.
	Suppose that \cref{eq:beta-mu-eta_p},
	and the following hold,
	\begin{subequations}\label{eq:beta-gamma-kappa-mu}
	\begin{gather}\label{eq:beta-gamma-kappa-mu:1}
		\upsilon(1+\kappa^2)\mu^{3/2}\eta_o^{1/2}\gamma^{1/2}\le \kappa,\\
\label{eq:beta-gamma-kappa-mu:2}
		\kappa\rho_{n}^{1/2}\le\varepsilon \quad \text{for $n\ge N_1$}
		.
	\end{gather}
	\end{subequations}
	If
	$V^{(0)}\in\sphere(\kappa/2)$,
	then for any $n\ge N_1$,
	\begin{align*}
		\prob{(\eventH_{n}\cap\eventH_o)^{\rmc}}
				\le 2nd\exp(-C_M  \lambda_1^2\nu^{-1}\gamma^{-2}\mu)
				+n(\ee d+p+1)\exp\left(-C_{\psi}\psi^{-1}(1\wedge\psi^{-1})\mu\right)
		,
	\end{align*}
	where 
	$\dps
	\eventH_n:= \bigcap_{r\le n}\eventT_r(\kappa), 
	\eventH_o:=\set*{\Nin{\varepsilon }\le N_1}.
	$

\end{lemma}
\subsection{Estimation with a Random Initial Guess}\label{ssec:estimation-with-a-random-initial-guess}
In order to compare our result with the previous results, we will estimate $ \N{\tan\Theta(\subspan(U^{(n)}),{\mathcal{U}}_*)}=\N{T^{(n)}}$.

First we give the estimation with a good initial guess in \Cref{lm:prop4:ppr1}.
\begin{lemma}\label{lm:prop4:ppr1}
	Suppose that \cref{eq:beta-mu-eta_p,eq:beta-gamma-kappa-mu} hold,
	and let $N_1=N_{\ln\varepsilon/\ln(\eta_*\gamma)}$.
	If $V^{(0)}\in\sphere(1)$,
then there exists a high-probability event $\eventH$ with $\prob{\eventH}\ge 1-\delta_1$,
	such that for any $n\ge N_1$,
	\begin{equation}\label{eq:lm:prop4:ppr1}
		\E{T^{(n)}\circ T^{(n)}\overevent \eventH}
			\le\prod_{r=1}^{n} \opL_r^2T^{(0)}\circ T^{(0)} +2\sum_{s=1}^{n}\eta_s^2\prod_{r=s+1}^{n} \opL_r^2H_{\circ} +R
			,
	\end{equation}
	where
	$\dps \N{R}_2 = C_R\frac{\varepsilon^2}{\ln\frac{nd}{\delta_1}}$ with $C_R$ an absolute constant, and
	$H_{\circ}%=\varc{\ul Y\ol Y^{\T}}
	\le 16\psi^4 H$ is as in \emph{\cref{itm:lm:diff-T:varcdT}} of \emph{\Cref{lm:diff-T}}.
\end{lemma}
Note that in \cref{eq:lm:prop4:ppr1} the inequality holds for each entry in the matrices.
In other words, the inequality in \cref{eq:lm:prop4:ppr1} represents $(d-p)\times p$ scalar inequalities.
Hence, in this sense, we can say that the iteration process is decoupled.
This observation is very useful for the gap-free consideration.

To deal with a random initial guess, a theorem by Huang et al.~\cite{huang2021streaming} is adopted.
\begin{lemma}[{\cite[Theorem~2.4]{huang2021streaming}}]\label{lm:-cite-theorem-2-4-huang2021streaming}
	If $\sup_{P^{\T}P=I_k}\N{P^{\T}(XX^{\T}-\Lambda)}_{\F}\le B$, and $\wtd V_0$ has i.i.d.\ standard Gaussian entries,
	writing 
	\[
		N_o=C_o\frac{pB^2}{\delta^2\gamma^2}(\ln\frac{dB}{\delta\gamma})^4,\qquad
		\eta_n=\eta_o=C'_o\frac{\ln(d/\delta)}{\gamma N_o},
	\]
	then with probability at least $1-\delta$, $\N{T^{(N_o)}}_2\le 1$.
\end{lemma}

The whole iteration process can be split into two parts: first the iteration goes from the initial guess into $\sphere(1)$, whose probability is estimated by \Cref{lm:-cite-theorem-2-4-huang2021streaming}; then the iteration goes from an approximation in $\sphere(1)$ to any precision we would like, whose probability is estimated by \Cref{lm:prop4:ppr1}.
Note that \Cref{lm:-cite-theorem-2-4-huang2021streaming} is only valid for a bounded distribution.
Thus we have to use it on the event $\eventQ_n(\mu)$ whose probability is bounded in \Cref{lm:quasi-bounded}.
\begin{theorem}\label{thm:main}
	If $U^{(0)}\in \mathbb{R}^{d\times p}$ has i.i.d\ standard Gaussian entries, 
		and $\eta_n=
		\begin{cases}
			C'_o\dfrac{\ln (d/\delta)}{\gamma N_o}, & n\le N_o\\
			C_\eta\dfrac{1}{\gamma n}, & n> N_o\\
		\end{cases}$
		for some $C_\eta\ge 1$,
	then there exists a high-probability event $\eventH_*$ with
	$\prob{\eventH_*}\ge 1- \delta$, such that	for any $n\ge N_o$,
		\begin{align*}
		\E{\N{{\tan\Theta(U^{(n)},U_*)}}_{\F}^2\given \eventH_*}
			%&\le C(d,n,\delta)\frac{\ln n}{n}\frac{1}{\lambda_p-\lambda_{p+1}}\sum_{j=1}^{p}\sum_{i=p+1}^{d} \frac{\lambda_j\lambda_i}{\lambda_j-\lambda_i} \\ &
		%\le C(d,p,n,\delta)\frac{\varphi(\Lambda)}{(2C_\eta-1)n},
		&\le \frac{64C_\eta\psi^4\varphi(\Lambda)}{1-\delta}\frac{n-N_o}{n^2}+ C\frac{pN_o^2}{(1-\delta)n^2}
		\\&\approx \frac{64C_\eta\psi^4\varphi(\Lambda)}{(1-\delta)n}
		,\qquad\qquad\text{as $n\to\infty$},
		\end{align*}
		where $C$ is an absolute constant,
		and $\psi$ is $\bX$'s Orlicz $\psi_2$ norm. %(defined below).
		Here
		\begin{equation}\label{eq:phi}
			\varphi(\Lambda)=\frac{1}{\gamma}\sum_{i=1}^{d-p}\sum_{j=1}^p\frac{\lambda_{p+i}\lambda_j}{\lambda_{p+i}-\lambda_j}
				\le\frac{ p(d-p)\lambda_p\lambda_{p+1}}{(\lambda_p-\lambda_{p+1})^2}
				.
		\end{equation}
\end{theorem}
\begin{proof}
	%The whole process can be split into two parts: first the iteration goes from the initial guess into $\sphere(1)$, whose probability is estimated by \Cref{lm:-cite-theorem-2-4-huang2021streaming}; then the iteration goes from an approximation in $\sphere(1)$ to any precision we would like, whose probability is estimated by \Cref{lm:prop4:ppr1}.
	The probability of the whole process is guaranteed by \Cref{lm:-cite-theorem-2-4-huang2021streaming,lm:prop4:ppr1,lm:quasi-bounded}.
	In the following, we will show the upper bound.

	Introduce $\Sum(A)$ for the sum of all the entries of $A$. In particular, $\Sum(A\circ A)=\N{A}_{\F}^2$.
	Write $ F_{D,2,2}^{(n',n)}(\gamma)=\sum_{s=n'}^{n}\eta_s^2\prod_{r=s+1}^{n} (1-\eta_r\gamma)^2$.
	By \Cref{lm:prop4:ppr1} with $\varepsilon=\frac{N_o}{n},\eta_n=\frac{C_\eta}{\gamma n}$, with high probability we have
	\begin{align*}
		\E{\N{T^{(n)}}_{\F}^2\overevent \eventH_*}
		&\le(F_*^{(N_o+1,n)})^2\N{T^{(N_o)}}_{\F}^2 +2\Sum(G\circ H_{\circ})+\Sum(R)
		%\\&\le p(F_*^{(N_o+1,n)})^2 +32\psi^4\lambda_{1\sim p}\lambda_{p+1\sim d}F_{D,2,2}^{(N_o+1,n)}+\OO\left(\frac{p\varepsilon^2}{\ln(nd/\delta)}\right)
		.
	\end{align*}
	where $G=[\gamma_{ij}]_{(d-p)\times p}$ with $\gamma_{ij}=F_{D,2,2}^{(N_o+1,n)}(\lambda_j-\lambda_{p+i})$.
	Then,
	\[
		F_*^{(N_o+1,n)}\le\left(1-C_\eta\frac{\frac{1}{N_o}+\dots+\frac{1}{n}}{n-N_o}\right)^{n-N_o}
		\le\left(1-C_\eta\frac{\ln n-\ln N_o}{n-N_o}\right)^{n-N_o}
		\le\ee^{-C_\eta(\ln n-\ln N_o)}=(\frac{N_o}{n})^{C_\eta},
	\]
	and
	\begin{align*}
		F_{D,2,2}^{(N_o+1,n)}(\lambda_j-\lambda_{p+i})
	&= 
	\begin{multlined}[t]
		\frac{C_\eta^2}{\gamma^2}\frac{\sum\limits_{s=1}^{n-N_o}(n-\nu_{ij})^2\dotsm(n-s+2-\nu_{ij})^2(n-s)^2\dotsm N_o^2}{n^2(n-1)^2\dotsm N_o^2}
	\\\qquad\text{where}\quad \nu_{ij}=C_\eta(\lambda_j-\lambda_{p+i})/\gamma
	\end{multlined}
	\\&\le \frac{C_\eta^2}{\gamma^2}\frac{\sum\limits_{s=1}^{n-N_o}(n-\floor{\nu_{ij}})^2\dotsm(n-s+2-\floor{\nu_{ij}})^2(n-s)^2\dotsm N_o^2}{n^2(n-1)^2\dotsm N_o^2}
	\\&= \frac{C_\eta^2}{\gamma^2n^2}\sum\limits_{s=1}^{n-N_o}\left(\frac{(n-\floor{\nu_{ij}})!(n-s)!}{(n-1)!(n-s+1-\floor{\nu_{ij}})!}\right)^2
	\\&\le \frac{C_\eta^2}{\gamma^2n^2}\sum\limits_{s=1}^{n-N_o}\left(\frac{n-s+1}{n}\right)^{2(\floor{\nu_{ij}}-1)}
	\\&= \frac{C_\eta^2}{\gamma^2n^{2\floor{\nu_{ij}}}}\left(\frac{n^{2\floor{\nu_{ij}}-1}-N_o^{2\floor{\nu_{ij}}-1}}{2\floor{\nu_{ij}}-1}+\frac{n^{2\floor{\nu_{ij}}-2}-N_o^{2\floor{\nu_{ij}}-2}}{2}+\OO(n^{2\floor{\nu_{ij}}-3})\right)
	\\&\le \frac{C_\eta^2(n-N_o)}{\gamma^2n^2(2\floor{\nu_{ij}}-1)}\left(1+C'\frac{N_o}{n}\right)
	\qquad \text{here $C'$ is an absolute constant}
	\\&\le \frac{2C_\eta^2(n-N_o)}{\gamma^2n^2\nu_{ij}}\left(1+C'\frac{N_o}{n}\right)
	\qquad \text{by $2\floor{\nu_{ij}}-1\ge \frac{\nu_{ij}}{2}$}
	\\&= \frac{2C_\eta(n-N_o)}{\gamma n^2(\lambda_j-\lambda_{p+i})}\left(1+C'\frac{N_o}{n}\right)
	.
	\end{align*}
	Thus, 
	\begin{align*}
		\E{\N{T^{(n)}}_{\F}^2\overevent \eventH_*}
		&\le \frac{pN_o^2}{n^2} +\frac{64C_\eta\psi^4}{\gamma}\sum_{i=1}^{d-p}\sum_{j=1}^p\frac{\lambda_{p+i}\lambda_j}{\lambda_{p+i}-\lambda_j}\frac{n-N_o}{n^2}\left(1+C'\frac{N_o}{n}\right)+C_R\frac{pN_o^2}{n^2\ln(nd/\delta)}
		\\&= 64C_\eta\psi^4\varphi(\Lambda)\frac{n-N_o}{n^2}+ \frac{pN_o^2}{n^2}\left(2+\frac{C_R}{\ln(nd/\delta)}\right)
		%\\&\to\frac{32\psi^4}{2C_\eta-1}\varphi(\Lambda)\frac{1}{n}
		%,\qquad\qquad\text{as $n\to\infty$}
		.
		\qedhere
	\end{align*}
\end{proof}
\subsection{Lower Bound and Gap-free Feature}\label{ssec:lower-bound}
Vu and Lei~\cite{vuL2013minimax} gives a minimax lower bound for the classical PCA, namely \cref{eq:minimax-bound}.
First we point out the factor $\frac{\lambda_1\lambda_{p+1}}{(\lambda_p-\lambda_{p+1})^2}$ in the lower bound
can be replaced by $\frac{\lambda_p\lambda_{p+1}}{(\lambda_p-\lambda_{p+1})^2}$.
The key to the point is a lemma there, which is presented in \Cref{lm:-cite-lemma-a-2-vul2013minimax-}.
\begin{lemma}[{\cite[Lemma~A.2]{vuL2013minimax}}]\label{lm:-cite-lemma-a-2-vul2013minimax-}
	Let $X_1,X_2\in \mathbb{R}^{d\times p}$ are column orthonormal matrices, $\beta\ge 0$,
	and
	\[
		\Sigma_i=I_p+\beta X_iX_i^{\T}\qquad\text{ for $i=1,2$}.
	\]
	If $\mathbb{P}_i$ are the $n$-fold product of the $N(0,\Sigma_i)$ probability measure, then
	the Kullback-Leibler (KL) divergence is 
	\[
		D(\mathbb{P}_1,\mathbb{P}_2)=\frac{n\beta^2}{1+\beta}\N{\sin\Theta(X_1,X_2)}_{\F}^2.
	\]
\end{lemma}
Then they use a substitution $\frac{\lambda_1\lambda_{p+1}}{(\lambda_p-\lambda_{p+1})^2}=\frac{1+\beta}{\beta^2}$ to obtain the lower bound.
However, as a matter of fact, in their construction on the covariance matrix $\Sigma_i$, it holds that $\lambda_1=\dots=\lambda_p$,
so when they proved the former bound, actually they were proving the latter bound.

Then we show the construction is also valid for the gap-free consideration.
\begin{lemma}\label{lm:-cite-lemma-a-2-vul2013minimax-:gap-free}
	Let $X_1,X_2\in \mathbb{R}^{d\times p},X_3\in \mathbb{R}^{d\times(q-p)}$ are column orthonormal matrices satisfying $X_1^{\T}X_3=X_2^{\T}X_3=0$, $q\ge p,\beta\ge 0,\wtd \beta\ge 0$,
	and
	\[
		\Sigma_1=I_p+\beta X_1X_1^{\T}+\wtd \beta X_3X_3^{\T},\quad
		\Sigma_2=I_p+\beta X_2X_2^{\T}+\beta X_3X_3^{\T}.
	\]
	If $\mathbb{P}_i$ are the $n$-fold product of the $N(0,\Sigma_i)$ probability measure, then
	the Kullback-Leibler (KL) divergence is 
	\[
		D(\mathbb{P}_1,\mathbb{P}_2)=\frac{n\beta^2}{1+\beta}\N{\sin\Theta(X_1,\begin{bmatrix}
				X_2&X_3
		\end{bmatrix})}_{\F}^2.
	\]
\end{lemma}
\begin{proof}
	The proof is similar as that of \cite[Lemma~A.2]{vuL2013minimax}. The key is to ensure
	\[
		\trace\left(\Sigma_2^{-1}(\Sigma_1-\Sigma_2)\right)=\frac{\beta^2}{1+\beta}\N{\sin\Theta(X_1,\begin{bmatrix}
					X_2&X_3
			\end{bmatrix})}_{\F}^2.
	\]
	In fact,
	\begin{align*}
		\MoveEqLeft[4]\trace\left(\Sigma_2^{-1}(\Sigma_1-\Sigma_2)\right)
		\\&= \trace\left([I_p+\beta X_2X_2^{\T}+\beta X_3X_3^{\T}]^{-1}\beta(X_1X_1^{\T}-X_2X_2^{\T})\right)
		\\&= \beta\trace\left(\left[\frac{1}{1+\beta}(X_2X_2^{\T}+X_3X_3^{\T})+I_p-X_2X_2^{\T}-X_3X_3^{\T}\right](X_1X_1^{\T}-X_2X_2^{\T})\right)
		\\&= \frac{\beta}{1+\beta}\trace\left(\left[(1+\beta)I_p-\beta(X_2X_2^{\T}+X_3X_3^{\T})\right](X_1X_1^{\T}-X_2X_2^{\T})\right)
		\\&= \frac{\beta^2}{1+\beta}\trace\left(\left[I_p-X_2X_2^{\T}-X_3X_3^{\T}\right](X_1X_1^{\T}-X_2X_2^{\T})\right)+\frac{\beta}{1+\beta}\trace(X_1X_1^{\T}-X_2X_2^{\T})
		\\&= \frac{\beta^2}{1+\beta}\trace\left(\left[I_p-X_2X_2^{\T}-X_3X_3^{\T}\right]X_1X_1^{\T}\right)+0
		\\&= \frac{\beta^2}{1+\beta}\N{\sin\Theta(X_1,\begin{bmatrix}
					X_2&X_3
			\end{bmatrix})}_{\F}^2.
			\qedhere
	\end{align*}
\end{proof}
Following \Cref{lm:-cite-lemma-a-2-vul2013minimax-:gap-free}, we can use a substitution $\frac{\lambda_p\lambda_{p+q}}{(\lambda_p-\lambda_{p+q})^2}=\frac{1+\beta}{\beta^2}$ to obtain the lower bound under the gap-free consideration, namely \Cref{thm:vu-lei-gap-free}.
\begin{theorem}[{Gap-free version of \cite[Theorem~3.1]{vuL2013minimax}}]\label{thm:vu-lei-gap-free}
	Let $\mathcal{P}_0(\sigma_*^2,d)$ be the set of all $d$-dimensional
sub-Gaussian distributions for which the eigenvalues of the covariance matrix satisfy
$\frac{\lambda_p\lambda_{q+1}}{(\lambda_p-\lambda_{q+1})^2}\le\sigma_*^2$.
Then for $1\le p\le q\le n$,
\begin{equation}\label{eq:minimax-bound:gap-free}
	\inf_{\dim\wtd{\mathcal{U}}_*=q}
	\sup_{\bX\in\mathcal{P}_0(\sigma_*^2,d)}
	\E{\N{\sin\Theta(\wtd{\mathcal{U}}_*,{\mathcal{U}}_*)}_{\F}^2}
	\ge cp(d-q)\frac{\sigma_*^2}{n}
	\ge c\frac{\lambda_p\lambda_{q+1}}{(\lambda_p-\lambda_{q+1})^2}\frac{p(d-q)}{n},
\end{equation}
where $c>0$ is an absolute constant.
\end{theorem}

On the other hand, the Oja's method can be easily proved to be gap-free, as is shown in \Cref{thm:main:gap-free},
which exactly matches the lower bound in \Cref{thm:vu-lei-gap-free}.
\begin{theorem}\label{thm:main:gap-free}
	If $U^{(0)}\in \mathbb{R}^{d\times q}$ has i.i.d\ standard Gaussian entries, 
	and $\eta_n$ are chosen as in \Cref{thm:main} with $\gamma$ replaced by $\wtd \gamma$,
	then \Cref{thm:main} still holds, except that $\varphi(\Lambda)$ in \cref{eq:phi} can be replaced by 
\[
	\wtd\varphi(\Lambda)=\frac{1}{\wtd\gamma}\sum_{i=1}^{d-q}\sum_{j=1}^p\frac{\lambda_{q+i}\lambda_j}{\lambda_{q+i}-\lambda_j}
	\le\frac{ p(d-q)\lambda_p(\lambda_p-\wtd\gamma)}{\wtd \gamma^2}.
\]
\end{theorem}
\begin{proof}
	The proof is nearly the same as that of \Cref{thm:main} except the bound.

	Note that  the iteration process produced by the Oja's method is decoupled, as is argued after \Cref{lm:prop4:ppr1}.
	Now what we need to estimate is 
	\begin{align*}
	\E{\N{\scrT_q(U^{(n)})}_{\F}^2\overevent \eventH_*}
		&\le
		\begin{multlined}[t]
			(F_*^{(N_o+1,n)})^2\N{\scrT_q(U^{(N_o)})}_{\F}^2 +32\psi^4\Sum(G_q\circ H_q)+\Sum(R)
	\\\qquad\qquad\text{where}\quad G_q=[\gamma_{ij}]_{(d-q)\times p}, H_q=[\eta_{ij}]_{(d-q)\times p}, R_q\in \mathbb{R}^{(d-q)\times p}
		\end{multlined}
		\\&\le \frac{pN_o^2}{n^2} +\frac{64C_\eta\psi^4}{\wtd\gamma}\sum_{i=1}^{d-q}\sum_{j=1}^p\frac{\lambda_{p+i}\lambda_j}{\lambda_{p+i}-\lambda_j}\frac{n-N_o}{n^2}\left(1+C'\frac{N_o}{n}\right)+C_R\frac{pN_o^2}{n^2\ln(nd/\delta)}
		%\\&\to\frac{32\psi^4}{2C_\eta-1}\varphi(\Lambda)\frac{1}{n}
		%,\qquad\qquad\text{as $n\to\infty$}
		.
		\qedhere
	\end{align*}
\end{proof}

\section{Proofs}\label{sec:proofs}
\subsection{Proof of \texorpdfstring{\Cref{lm:diff-T}}{Lemma~3.1}}\label{ssec:proof-of-lm:diff-T}
%\begin{proof}[Proof of \Cref{lm:diff-T}]
For readability,
we will drop the superscript ``$\cdot^{(n)}$'', and use the superscript ``$\cdot\new$''
to replace ``$\cdot^{(n+1)}$'' for $V,T,R_E$, drop
	the superscript ``$\cdot^{(n+1)}$'' on $Y,Z$ and the subscript ``$\cdot_{n+1}$'' on $\eta$,
	and drop the conditional sign ``$\given\fil_n$'' in the computation of $\E{\cdot},\var(\cdot),\cov(\cdot)$
	with the understanding that they are conditional with respect to $\fil_n$.
Finally, for any expression or variable $*$, we define $\Delta *:=*\new-*$.
	
	Consider \cref{itm:lm:diff-T:well-defined}.
	Since $n<\Nout{\kappa }$, we have $V\in\sphere(\kappa )$ and $\tau=\N{T}_2\le\kappa$.
	Thus, $\N{\ol V^{-1}}_2\le \sqrt{\kappa^2+1} $ and $T=\ul V\ol V^{-1}$ is well-defined.
	%Partition
	%\[
	%Y=\kbordermatrix{ &\scs 1\\
		   %\scs p   & \ol Y \\
		   %\scs d-p & \ul Y}.
%%	\begin{blockarray}{cc}
%%		\scs 1 & \\
%%		\begin{block}{[>{\,}c<{\,}\Right{]}{,}>{\scs }c}
%%			\ol Y & p \\
%%			\ul Y & d-p \\
%%		\end{block}
%%	\end{blockarray}
%%	\begin{blockarray}{cc}
%%		\scs 1 & \\
%%		\begin{block}{[>{\,}c<{\,}\Right{]}{.}>{\scs }c}
%%			\ol R & p \\
%%			\ul R & d-p \\
%%		\end{block}
%%	\end{blockarray}
	%\]
	We have
	$\ol V\new = (\ol V + \eta \ol Y Z^{\T})S$,
	where $S$ is nonsingular as is stated above.
	According to the Sherman-Morrison formula%\cite[p.~95]{demm:1997}
	, we get $\ol V + \eta\ol YZ^{\T}$ or $\ol V\new$ is nonsingular, if and only if
	$1+\eta \xi \ol Y\ne 0$ where $\xi:=Z^{\T}\ol V^{-1}\ol Y$,
	and 
	\begin{equation*}\label{eq:olVinv}
		(\ol V\new)^{-1}=S^{-1}\left(I-\frac{\eta}{1+\eta\xi}\ol V^{-1}\ol YZ^{\T}\right)\ol V^{-1}
		.
	\end{equation*}
	Since $n<\Nqb{\mu }$, we have $\N{Z}_2\le \lambda_{1\sim p}^{1/2}\mu^{1/2}, \N{\ol Y}_2\le \lambda_{1\sim p}^{1/2}\mu^{1/2}$.
	By \cref{eq:beta-mu-eta_p},
	we find
	\[
		\abs{Z^{\T}\ol V^{-1}\ol Y}\le
		\N{Z}_2\N{\ol V^{-1}}_2\N{\ol Y}_2\le\lambda_{1\sim p}\mu\sqrt{\kappa^2+1}\le\frac{1}{2\eta}.
	\]
Hence $T\new=\ul V\new(\ol V\new)^{-1}$ is well-defined. This proves \cref{itm:lm:diff-T:well-defined}.

	For \cref{itm:lm:diff-T:normdT}, using the Sherman-Morrison-Woodbury formula, we get
	\begin{align*}
		\Delta T &= \ul V\new(\ol V\new)^{-1}-\ul V\ol V^{-1}
		\\ &=(\ul V + \eta \ul YZ^{\T})\left(I-\frac{\eta}{1+\eta\xi}\ol V^{-1}\ol YZ^{\T}\right)\ol V^{-1}-\ul V\ol V^{-1}
		\\ &=\left(\eta \ul YZ^{\T}-\frac{\eta}{1+\eta\xi}\ul V\ol V^{-1}\ol YZ^{\T}-\frac{\eta^2}{1+\eta\xi}\ul YZ^{\T}\ol V^{-1}\ol YZ^{\T}\right)\ol V^{-1}
		\\ &=\eta \left(\ul Y-\frac{1}{1+\eta\xi}T\ol Y-\frac{\eta\xi}{1+\eta\xi}\ul Y\right)Z^{\T}\ol V^{-1}
		\\ &=\frac{\eta }{1+\eta\xi}\left(\ul Y-T\ol Y\right)Y^{\T}V\ol V^{-1}
		\\ &=\frac{\eta }{1+\eta\xi}T_lYY^{\T}T_r
		,
	\end{align*}
	where
	$ T_l=\begin{bmatrix} -T & I\end{bmatrix}$ and $T_r=\begin{bmatrix} I \\ T \end{bmatrix}$.
	Note that 
	\begin{equation}\label{eq:dT:mainpart}
		T_lYY^{\T}T_r=\ul Y\ol Y^{\T}-T\ol Y\ul Y^{\T}T-T\ol Y\ol Y^{\T}+\ul Y\ul Y^{\T}T.
	\end{equation}
	For any positive semi-definite matrices $A_1,A_2$ and a matrix $X$, by \cite{kittaneh2007inequalities}
	\[
		\N{A_1X-XA_2}_2\le\max\set{\N{A_1}_2,\N{A_2}_2}\N{X}_2.
	\]
	Thus,
	\begin{align*}
		\N{T_lYY^{\T}T_r}_2
		&\le\lambda_{p+1\sim d}^{1/2}\lambda_{1\sim p}^{1/2}\mu(1+\tau^2)+(\lambda_{1\sim p}\vee\lambda_{p+1\sim d})\mu\tau
		\\&=\mu[\nu^{1/2}\lambda_{1\sim p}(1+\tau^2)+\nu_1\lambda_{1\sim p}\tau],
	\end{align*}
	and then
	\begin{align*}
		\N{\Delta T}_2
		&\le 2\mu\eta[\nu^{1/2}\lambda_{1\sim p}(1+\tau^2)+\nu_1\lambda_{1\sim p}\tau]
			.
	\end{align*}

	Consider \cref{itm:lm:diff-T:dT}.
	Clearly $T_lV=0$ and $V=T_r\ol V,\ol V^{-1}=V^{\T}T_r,VV^{\T}T_r=T_r$.
	Write 
	\begin{equation*}\label{eq:diff-T}
			\Delta T
			=T_l(\eta  YY^{\T}+R_T)T_r
			,
	\end{equation*}
	where 
	\[
		R_T 
		=-\eta \frac{\eta\xi}{1+\eta \xi}YY^{\T}
		=-\eta \frac{\eta Z^{\T}\ol V^{-1}\ol Y}{1+\eta Z^{\T}\ol V^{-1}\ol Y}YY^{\T}
		,
	\]
	and
	\begin{equation}\label{eq:normR_T}
		\N{T_lR_TT_r}_2\le 2\lambda_{1\sim p}\mu^2\eta^2(1+\tau^2)^{1/2} [\nu^{1/2}\lambda_{1\sim p}(1+\tau^2)+\nu_1\lambda_{1\sim p}\tau].
	\end{equation}
	In \cref{eq:dT:mainpart},
\begin{subequations}\label{eq:dT:mainpart:E}
\begin{alignat*}{2}
\E{\ul Y\ol Y^{\T}}&=0,
		&\quad&\E{T\ol Y\ol Y^{\T}}=T\E{\ol Y\ol Y^{\T}}=T\ol\Lambda , 	%\label{eq:dT:mainpart:E-1}
		\\
\E{T\ol Y\ul Y^{\T}T}=T\E{\ol Y\ul Y^{\T}}T&=0,
		&\quad&\E{\ul Y\ul Y^{\T}T}=\E{\ul Y\ul Y^{\T}}T=\ul\Lambda T. \label{eq:dT:mainpart:E-2}
\end{alignat*}
\end{subequations}
	Thus,
	$\E{\Delta T}
		=\eta  (\ul\Lambda T-T\ol\Lambda )+R_E$,
	where $R_E=\E{T_lR_TT_r}$.
	Therefore,
	\begin{align*}
		\N{R_E}_2
		&=\N*{\E{T_l\left[-\eta \frac{\eta Z^{\T}\ol V^{-1}\ol Y}{1+\eta Z^{\T}\ol V^{-1}\ol Y}YY^{\T}\right]T_r}}_2
		\\&\le\max_{Y}\abs*{-\eta \frac{\eta Z^{\T}\ol V^{-1}\ol Y}{1+\eta Z^{\T}\ol V^{-1}\ol Y}}\N*{\E{T_lYY^{\T}T_r}}_2
		\\&\le 2\eta^2\lambda_{1\sim p}^{1/2}\mu^{1/2}(1+\tau^2)^{1/2}\lambda_{1\sim p}^{1/2}\mu^{1/2}\lambda_1\tau
		\\&= 2\lambda_1\lambda_{1\sim p}\mu\eta^2\tau(1+\tau^2)^{1/2}.
	\end{align*}

	Now we turn to \cref{itm:lm:diff-T:estimate}. We have
\begin{equation}\label{eq:itm:lm:diff-T:estimate:pf-1}
		\varc{\Delta T}=
		\varc{T_l(\eta  YY^{\T}+R_T)T_r}
		=\eta ^2 \varc{T_lYY^{\T}T_r}
		+2\eta  R_{\circ,1}+ R_{\circ,2},
\end{equation}
	where $R_{\circ,1}=\covc{T_lYY^{\T}T_r,T_lR_TT_r}$, and $R_{\circ,2}=\varc{T_lR_TT_r}$.
	By \cref{eq:dT:mainpart},
\begin{equation}\label{eq:itm:lm:diff-T:estimate:pf-2}
		\varc{T_lYY^{\T}T_r}=\varc{\ul Y\ol Y^{\T}}+R_{\circ,0},
\end{equation}
	where
	\begin{align*}
	R_{\circ,0}
		&=\varc{T\ol Y\ul Y^{\T}T}-2\covc{\ul Y\ol Y^{\T},T\ol Y\ul Y^{\T}T}+\varc{T\ol Y\ol Y^{\T}-\ul Y\ul Y^{\T}T} 	\\
		&\qquad -2\covc{\ul Y\ol Y^{\T},T\ol Y\ol Y^{\T}-\ul Y\ul Y^{\T}T}
		+2\covc{T\ol Y\ul Y^{\T}T,T\ol Y\ol Y^{\T}-\ul Y\ul Y^{\T}T}
		.
	\end{align*}
Examine \cref{eq:itm:lm:diff-T:estimate:pf-1} and \cref{eq:itm:lm:diff-T:estimate:pf-2} together to get
$H_{\circ}=\varc{\ul Y\ol Y^{\T}}$ and  $R_{\circ}^{(n)}=\eta ^2R_{\circ,0}+2\eta  R_{\circ,1}+R_{\circ,2}$.
We note
\begin{gather*}
Y_j=e_j^{\T}Y
	=e_j^{\T}\Lambda^{1/2}\Lambda^{-1/2}Y
	=\lambda_j^{1/2}e_j^{\T}\Lambda^{-1/2}Y, \\
e_i^{\T}\varc{\ul Y\ol Y^{\T}}e_j
		=\var(e_i^{\T}\ul Y\ol Y^{\T}e_j)=\var(Y_{p+i}Y_{j})
		=\E{Y_{p+i}^2Y_{j}^2}.
\end{gather*}
By \cref{itm:vershynin:5.11} of \Cref{lm:-cite-lemma-5-5-and-5-10-5-12-vershynin2012introduction-}, %\cite[(5.11)]{vershynin2012introduction},
	\[
		\E{Y_j^4}
		= \lambda_j^2\E{(e_j^{\T}\Lambda^{-1/2}Y)^4}
		\le 16\lambda_j^2\N{e_j^{\T}\Lambda^{-1/2}Y}_{\psi_2}^4
		\le 16\lambda_j^2\N{\Lambda^{-1/2}Y}_{\psi_2}^4
		= 16\lambda_j^2\psi^4.
\]
Therefore
	\[
		e_i^{\T}\varc{\ul Y\ol Y^{\T}}e_j
		\le
		[\E{Y_{p+i}^4} \E{Y_j^4}]^{1/2}
		\le 16\lambda_{p+i}\lambda_j\psi^4,
	\]
i.e., $H_{\circ}=\varc{\ul Y\ol Y^{\T}}\le 16\psi^4 H$.
This proves \cref{itm:lm:diff-T:varcdT}.
To show \cref{itm:lm:diff-T:normR_H}, first we bound the entrywise variance and covariance.
For any matrices $A_1,A_2$, by Schur's inequality
(which was generalized to all unitarily invariant norm in \cite[Theorem~3.1]{hornM1990analog}),
		\begin{equation}\label{eq:schur-ineq-UI}
			\N{A_1\circ A_2}_2\le\N{A_1}_2\N{A_2}_2,
		\end{equation}
we have
\begin{subequations}\label{eq:covc&varc}
\begin{align*}
\N{\covc{A_1,A_2}}_2
	&=\N{\E{A_1\circ A_2}-\E{A_1}\circ\E{A_2}}_2 \nonumber\\
	&\le\E{\N{A_1\circ A_2}_2}+\N{\E{A_1}\circ\E{A_2}}_2 \nonumber\\
	&\le \E{\N{A_1}_2\N{A_2}_2}+ \N{\E{A_1}}_2\N{\E{A_2}}_2, %\label{eq:covc}
	\\
\N{\varc{A_1}}_2
	&\le \E{\N{A_1}_2^2}+ \N{\E{A_1}}_2^2. %\label{eq:varc}
\end{align*}
\end{subequations}
Apply \cref{eq:covc&varc} to $R_{\circ,1}$ and $R_{\circ,2}$ with \cref{eq:normR_T} to get
\begin{subequations}\label{eq:diff-T:pf-10}
\begin{align*}
	\N{R_{\circ,1}}_2&\le\N{T_lYY^{\T}T_r}_2\N{T_lR_TT_r}_2+\N{T_l\Lambda T_r}_2\N{R_E}_2\nonumber\\
	&\le
		2\lambda_{1\sim p}\mu^3\eta^2(1+\tau^2)^{1/2} [\nu^{1/2}\lambda_{1\sim p}(1+\tau^2)+\nu_1\lambda_{1\sim p}\tau]^2+
		2\lambda_1^2\lambda_{1\sim p}\mu\eta^2\tau^2(1+\tau^2)^{1/2}\nonumber
		\\&\le
		4\nu\lambda_{1\sim p}^3\mu^3\eta^2(1+\tau^2)^{1/2} \left[1+\tau^2+\nu_1\nu^{-1/2}\tau\right]^2,\\
	\N{R_{\circ,2}}_2&\le 2(\N{T_lR_TT_r}_2)^2 \nonumber
	\\&\le
8\lambda_{1\sim p}^2\mu^4\eta^4(1+\tau^2) [\nu^{1/2}\lambda_{1\sim p}(1+\tau^2)+\nu_1\lambda_{1\sim p}\tau]^2\nonumber\\
&\le8\nu\lambda_{1\sim p}^4\mu^4\eta^4(1+\tau^2) \left[1+\tau^2+\nu_1\nu^{-1/2}\tau\right]^2
.
\end{align*}
\end{subequations}
For $R_{\circ,0}$, 	by \cref{eq:dT:mainpart:E}, we have
\begin{align*}
		\N{\varc{T\ol Y\ul Y^{\T}T}}_2&\le\E{\N{\ul Y\ol Y^{\T}}_2^2}\N{T}_2^4,\\
		\N{\covc{\ul Y\ol Y^{\T},T\ol Y\ul Y^{\T}T}}_2&\le\E{\N{\ul Y\ol Y^{\T}}_2^2}\N{T}_2^2,\\
		\N{\covc{\ul Y\ol Y^{\T},T\ol Y\ol Y^{\T}-\ul Y\ul Y^{\T}T}}_2&\le\E{\N{\ul Y\ol Y^{\T}}_2\N{T\ol Y\ol Y^{\T}-\ul Y\ul Y^{\T}T}_2},\\
		\N{\covc{T\ol Y\ul Y^{\T}T,T\ol Y\ol Y^{\T}-\ul Y\ul Y^{\T}T}}_2&\le\E{\N{\ul Y\ol Y^{\T}}_2\N{T\ol Y\ol Y^{\T}-\ul Y\ul Y^{\T}T}_2}\N{T}_2^2,\\
		\N{\varc{T\ol Y\ol Y^{\T}-\ul Y\ul Y^{\T}T}}_2&\le\E{\N{T\ol Y\ol Y^{\T}-\ul Y\ul Y^{\T}T}_2^2}+\N{T\ol\Lambda-\ul\Lambda T}_2^2.
		%\le(\lambda_{1\sim p}\vee\lambda_{p+1\sim d}+\lambda_1)^2\N{T}_2^2,\\
\end{align*}
Since
\begin{align*}
	%\N{\ol Y\ol Y^{\T}}_2+\N{\ul Y\ul Y^{\T}}_2&=\ol Y^{\T}\ol Y+\ul Y^{\T}\ul Y=Y^{\T}Y\le \lambda_{1\sim d}\mu, \\
	\N{\ul Y\ol Y^{\T}}_2&=(\ol Y^{\T}\ol Y)^{1/2}(\ul Y^{\T}\ul Y)^{1/2}
	\le \nu^{1/2}\lambda_{1\sim p}\mu,
	%\le\frac{\ol Y^{\T}\ol Y+\ul Y^{\T}\ul Y}{2}
	%\le \frac{\lambda_{1\sim d}\mu }{2},
\end{align*}
	we have
\begin{align}
\N{R_{\circ,0}}_2
&\le \nu\lambda_{1\sim p}^2\mu^2\tau^4+2\nu\lambda_{1\sim p}^2\mu^2\tau^2+\nu_1^2\lambda_{1\sim p}^2\mu^2\tau^2+\lambda_1^2\tau^2\nonumber
\\&\qquad + 2\nu^{1/2}\lambda_{1\sim p}\nu_1\lambda_{1\sim p}\mu^2\tau+2\nu^{1/2}\lambda_{1\sim p}\nu_1\lambda_{1\sim p}\mu^2\tau^3\nonumber
\\	&\le 2\nu_1\nu^{1/2}\lambda_{1\sim p}^2\mu^2\tau \Big(1+\left[1+\nu_1\nu^{-1/2}\right]\tau +\tau ^2+\frac{1}{2}\tau ^3\Big). \label{eq:diff-T:pf-11}
\end{align}
Finally collecting \cref{eq:diff-T:pf-10} and \cref{eq:diff-T:pf-11} yields the desired bound on
$R_{\circ}^{(n)}=\eta ^2R_{\circ,0}+2\eta  R_{\circ,1}+R_{\circ,2}$.
%\end{proof}

\subsection{Two probability estimations}\label{ssec:two-probability-estimations}

\begin{lemma}\label{lm:quasi-bounded}
	For any $n\ge 1$,
	\[
		\prob{\Nqb{\mu }>n}\ge 1-
		n(\ee d+p+1)\exp\left(-C_{\psi}\psi^{-1}(1\wedge\psi^{-1})\mu\right),
	\]
	where $C_\psi$ is an absolute constant.
\end{lemma}

\begin{proof}
	Since
	\[
		\set*{\Nqb{\mu }\le n}
		\subset \bigcup_{n\le n}\left(\set*{\N{Z^{(n)}}_2\ge \lambda_{1\sim p}^{1/2}\mu^{1/2} }\cup\bigcup_{1\le i\le d}\set*{\abs{e_i^{\T}Y^{(n)}}\ge \lambda_i^{1/2}\mu^{1/2} } \right)
		,
	\]
	we know
\begin{equation}\label{eq:add2pf-1}
		\prob{\Nqb{\mu }\le n}
		\le \sum_{n\le n}\left(\prob{\N{Z^{(n)}}_2\ge \lambda_{1\sim p}^{1/2}\mu^{1/2} }+\sum_{1\le i\le d}\prob{\abs{e_i^{\T}Y^{(n)}}\ge \lambda_i^{1/2}\mu^{1/2} }\right)
		.
\end{equation}
	First,
	\begin{align*}
		\prob{\abs{e_i^{\T}Y^{(n)}}\ge \lambda_i^{1/2}\mu^{1/2} }
		&= \prob{\abs*{\frac{(\Lambda^{1/2}e_i)^{\T}}{\N{\Lambda^{1/2} e_i}_2}\Lambda^{-1/2}Y^{(n)}}\ge \frac{\lambda_i^{1/2}\mu^{1/2} }{\N{\Lambda^{1/2} e_i}_2}}
					   \nonumber \\
					   &\le \exp\left(1-\frac{C_{\psi,i}\frac{\lambda_i\mu }{e_i^{\T}\Lambda e_i}}{\N{\frac{(\Lambda^{1/2}e_i)^{\T}}{\N{\Lambda^{1/2} e_i}_2}\Lambda^{-1/2}Y^{(n)}}_{\psi_2}}\right)
			\qquad\qquad
			\text{by \cref{itm:vershynin:5.10} of \Cref{lm:-cite-lemma-5-5-and-5-10-5-12-vershynin2012introduction-}} \nonumber 	\\
		%\text{by \cite[(5.10)]{vershynin2012introduction}} \nonumber 	\\
		&\le \exp\left(1-\frac{C_{\psi,i}\lambda_i\mu }{\N{\Lambda^{-1/2}Y^{(n)}}_{\psi_2}\lambda_i}\right)
		=\exp\left(1-C_{\psi,i}\psi^{-1}\mu \right)		, \label{eq:add2pf-2}
	\end{align*}
	where $C_{\psi,i},i=1,\dots,d$ are absolute constants. %\cite[(5.10)]{vershynin2012introduction}.
	Next, we claim
	\begin{equation}\label{eq:norm-Z-claim}
		\prob{\N{Z^{(n)}}_2\ge \lambda_{1\sim p}^{1/2}\mu^{1/2} }\le
		(p+1)\exp\left(-C_{\psi,d+1}\psi^{-2}\mu \right)
		.
	\end{equation}
	Together, \cref{eq:add2pf-1} -- \cref{eq:norm-Z-claim} yield
	\begin{align*}
		\prob{\Nqb{\mu }\le n}
		&=
		\sum_{n\le n}\sum_{1\le i\le d} \exp\left(1-C_{\psi,i}\psi^{-1}\mu\right)
		+\sum_{n\le n} (p+1)\exp\left(-C_{\psi,d+1}\psi^{-2}\mu\right)
		\\ &\le n(\ee d+p+1)\exp\left(-C_{\psi}\psi^{-1}(1\wedge\psi^{-1})\mu\right),
	\end{align*}
	where $C_\psi=\min_{1\le i\le d+1}C_{\psi,i}$.
	Finally, use $\prob{\Nqb{\mu }> n}=1-\prob{\Nqb{\mu }\le n}$ to complete the proof.

	It remains to prove the claim \cref{eq:norm-Z-claim}.
	To avoid the cluttered superscripts, we drop the superscript ``$\cdot^{(n-1)}$'' on $V$, and drop the superscript ``$\cdot^{(n)}$'' on $Y,\,Z$.
	Consider
	\[
		W:=\begin{bmatrix}
			0 & Z\\ Z^{\T} & 0
		\end{bmatrix}
		= \begin{bmatrix}
			 & V^{\T}Y\\ Y^{\T}V &
		\end{bmatrix}
		=\sum_{k=1}^{d} Y_k\begin{bmatrix}
			& & & v_{k1} \\
			& & & \vdots \\
			& & & v_{kp} \\
			v_{k1} & \cdots & v_{kp} & 0\\
		\end{bmatrix}
		=: \sum_{k=1}^{d} Y_kW_k,
	\]
	where $v_{ij}$ is the $(i,j)$-entry of $V$.
	By the matrix version of master tail bound \cite[Theorem~3.6]{tropp2012user}, for any $\xi>0$,
	\[
		\prob{\N{Z}_2\ge \xi}
		= \prob{\lambda_{\max}(W)\ge \xi}
		\le \inf_{\theta>0}\ee^{-\theta \xi}\trace\exp\left(\sum_{k=1}^{d}\ln \E{\exp(\theta Y_kW_k)} \right).
	\]
	$Y$ is sub-Gaussian and $\E{Y}=0$, and so is $Y_k$.
	Moreover,
	\[
		\N{Y_k}_{\psi_2}=\N{e_k^{\T}\Lambda^{1/2}}_2\N*{\frac{e_k^{\T}\Lambda^{1/2}}{\N{e_k^{\T}\Lambda^{1/2}}_2}\Lambda^{-1/2}Y}_{\psi_2}
		\le \lambda_k^{1/2}\N{\Lambda^{-1/2}Y}_{\psi_2}
		=\lambda_k^{1/2}\psi.
	\]
	Also, by \cref{itm:vershynin:5.12} of \Cref{lm:-cite-lemma-5-5-and-5-10-5-12-vershynin2012introduction-}, %\cite[(5.12)]{vershynin2012introduction},
	\[
		\E{\exp(\theta W_kY_k)}\le \exp(C_{\psi,d+k}\theta^2W_k\circ W_k\N{Y_k}_{\psi_2}^2)
		\le \exp(c_{\psi,k}\theta^2\lambda_k\psi^2W_k\circ W_k)
		,
	\]
	where $c_{\psi,k}, k=1,\dots,d$ are absolute constants.
	Therefore, writing $[4C_{\psi,d+1}]^{-1}=\max_{1\le k\le d}c_{\psi,k}$
	and $W_\psi:=\sum_{k=1}^{d} \lambda_kW_k\circ W_k$ with the spectral decomposition $W_\psi=V_\psi\Lambda_\psi V_\psi^{\T}$, we have
	\begin{align*}
		\trace\exp\left(\sum_{k=1}^{d}\ln \E{\exp(\theta Y_kW_k)} \right)
		&\le \trace\exp\left(\sum_{k=1}^{d} c_{\psi,k}\theta^2\lambda_k\psi^2W_k\circ W_k\right)
		\\&\le \trace\exp([4C_{\psi,d+1}]^{-1}\theta^2\psi^2W_\psi)
		\\&= \trace\exp([4C_{\psi,d+1}]^{-1}\theta^2\psi^2V_\psi\Lambda_\psi V_\psi^{\T})
		\\&= \trace\left( V_\psi \exp([4C_{\psi,d+1}]^{-1}\theta^2\psi^2\Lambda_\psi) V_\psi^{\T}\right)
		\\&= \trace\exp([4C_{\psi,d+1}]^{-1}\theta^2\psi^2\Lambda_\psi)
		\\&\le (p+1)\exp([4C_{\psi,d+1}]^{-1}\theta^2\psi^2\lambda_{\max}(\Lambda_\psi))
		\\&= (p+1)\exp([4C_{\psi,d+1}]^{-1}\theta^2\psi^2\lambda_{\max}(W_\psi))
		.
	\end{align*}
	Note that
	\[
		W_\psi
		=
		\begin{bmatrix}
			0&\cdots&0&\dps\sum_{k=1}^{d} \lambda_kv_{k1}^2 \\
			\vdots&&\vdots&\vdots \\
			0&\cdots&0&\dps\sum_{k=1}^{d} \lambda_kv_{kp}^2 \\
			\dps\sum_{k=1}^{d} \lambda_kv_{k1}^2 & \cdots &\dps\sum_{k=1}^{d} \lambda_kv_{kp}^2 & 0 \\
		\end{bmatrix}
		=
		\begin{bmatrix}
			0&\cdots&0&e_1^{\T}V^{\T}\Lambda Ve_1 \\
			\vdots&&\vdots&\vdots \\
			0&\cdots&0&e_p^{\T}V^{\T}\Lambda Ve_p \\
			e_1^{\T}V^{\T}\Lambda Ve_1 &\cdots & e_p^{\T}V^{\T}\Lambda Ve_p &0 \\
		\end{bmatrix},
	\]
	and thus
	\begin{align*}
		\lambda_{\max}(W_\psi)=\N*{\begin{bmatrix}
			e_1^{\T}V^{\T}\Lambda Ve_1 \\
			\vdots \\
			e_p^{\T}V^{\T}\Lambda Ve_p \\
		\end{bmatrix}}_2
		\le \sum_{k=1}^{p} e_k^{\T}V^{\T}\Lambda Ve_k
		&= \trace(V^{\T}\Lambda V)
		\\&\le \max_{V^{\T}V=I_p} \trace(V^{\T}\Lambda V)
		= \sum_{k=1}^{p}\lambda_k
		=\lambda_{1\sim p}.
	\end{align*}
	In summary, we have
	\begin{align*}
		\prob{\N{Z}_2\ge \xi}
		&\le (p+1)\inf_{\theta>0}\exp([4C_{\psi,d+1}]^{-1}\theta^2\psi^2\lambda_{1\sim p}-\theta \xi)
		\\&= (p+1)\exp\left(-\frac{C_{\psi,d+1}\xi^2}{\psi^2\lambda_{1\sim p}}\right).
	\end{align*}
	Substituting $\xi=\lambda_{1\sim p}^{1/2}\mu^{1/2} $, we have the claim \cref{eq:norm-Z-claim}.
\end{proof}
\begin{lemma}\label{lm:prop3:ppr1}
	Suppose that \cref{eq:beta-mu-eta_p} holds.
	If $n-1<{\Nqb{\mu }\wedge \Nout{\kappa }}$
	and $T^{(0)}$ is well-defined, then
	\begin{equation*}\label{eq:lm:prop3:ppr1}
		\prob{\eventM_n(\kappa,\mu) }\ge 1- 2d \exp(-C_M\lambda_1^2\nu^{-1}\gamma^{-2}\mu ),
	\end{equation*}
	where $C_M =\frac{C_{D,2,1}^2}{8C_{D,2,2}}$.
\end{lemma}
\begin{proof}
	For any $n-1<{\Nqb{\mu }\wedge \Nout{\kappa }}$,
	$V^{(n-1)}\in\sphere(\kappa  )$ and thus
	$\N{T^{(n-1)}}_2\le \kappa $  by \cref{eq:sphereK-normT}.
	Therefore, by \cref{itm:lm:diff-T:normdT} of \Cref{lm:diff-T}, we have
	\begin{equation*}\label{eq:norm-D}
		\begin{aligned}
			\N{D^{(n)}}_2
			&=\N*{{T^{(n)}-T^{(n-1)}-\E{T^{(n)}-T^{(n-1)}\given\fil_n}}}_2
			\\ &\le \N{T^{(n)}-T^{(n-1)}}_2 +\E{\N{T^{(n)}-T^{(n-1)}}_2 \given\fil_n}
			\\ &\le 4\mu\eta_n[\nu^{1/2}\lambda_{1\sim p}(1+\tau^2)+\nu_1\lambda_{1\sim p}\tau]
			.
		\end{aligned}
	\end{equation*}
	For any $n-1<{\Nqb{\mu }\wedge \Nout{\kappa}}$,
	\begin{align*}
		\N{T_R^{(n)}}_2
		&\le\sum_{s=1}^{n}\left(\prod_{r=s+1}^{n} \N{\opL_r}_2\right)\N{R_E^{(s-1)}}_2
		\\&\le\sum_{s=1}^{n}2\lambda_1\lambda_{1\sim p}\kappa(1+\kappa^2)^{1/2}\mu\eta_{s-1}^2 \prod_{r=s+1}^{n} \N{\opL_r}_2
		\\ &\le 2\lambda_1\lambda_{1\sim p}\kappa(1+\kappa^2)^{1/2}F_{D,2,1}^{(1,n)}
		\\ &\le
		2C_{D,2,1}\lambda_1\lambda_{1\sim p}\kappa(1+\kappa^2)^{1/2}\mu\eta_n\gamma^{-1}
		\\ &\le 
		2C_{D,2,1}\lambda_1\lambda_{1\sim p}(1+\kappa^2)\mu^{3/2}\eta_n^{1/2}\gamma^{-3/2}
\qquad\qquad\text{by $\eta_n\gamma\le 1$}.
	\end{align*}
	Similarly,
	\begin{align*}
	\N{T_D^{(n)}}_2
	&\le \sum_{s=1}^{n}\left(\prod_{r=s+1}^{n} \N{\opL_r}_2\right)\N{D^{(s)}}_2
	\\&\le \sum_{s=1}^{n} 4[\nu^{1/2}\lambda_{1\sim p}(1+\kappa^2)+\nu_1\lambda_{1\sim p}\kappa]\mu\eta_s\prod_{r=s+1}^{n} \N{\opL_r}_2
	\\&\le 4[\nu^{1/2}\lambda_{1\sim p}(1+\kappa^2)+\nu_1\lambda_{1\sim p}\kappa]\mu F_{D,1,1}^{(1,n)}
	\\&\le 4C_{D,1,1}[\nu^{1/2}\lambda_{1\sim p}(1+\kappa^2)+\nu_1\lambda_{1\sim p}\kappa]\mu\gamma^{-1}
		.
    \end{align*}
	Also,
	$
		\N{T_*^{(n)}}_2
		\le F_*^{(1,n)}\N{T^{(0)}}_2
		\le \N{T^{(0)}}_2.
		%\le \kappa .
	$
	For fixed $n>0$ and $\eta_n>0$,
	\[
		\set*{
			M_0^{(n)}:= \prod_{r=1}^{n} \opL_r T^{(0)},M_t^{(n)}:=\prod_{r=1}^{n} \opL_r T^{(0)}+\sum_{s=1}^{{t\wedge(\Nout{\kappa }-1)}}\prod_{r=s+1}^{n}\opL_r D^{(s)}
			 :  t=1,\dots,n
		}
	\]
	forms a martingale with respect to $\fil_t$,
	because
	\[
		\E{\N{M_t^{(n)}}_2 }
		\le \N{T_*^{(n)}}_2 +\N{T_D^{(n)}}_2 <+\infty,
	\]
	and
	\[
		\E{M_{t+1}^{(n)}-M_{t}^{(n)}\given \fil_t}
		=\E{ \prod_{r=t+2}^{n}\opL_r D^{(t+1)}\given \fil_t}
		= \prod_{r=t+2}^{n}\opL_r \E{D^{(t+1)}\given \fil_t}=0.
	\]
	Use the matrix version of Azuma's inequality \cite[Section~7.2]{tropp2012user} to get,
	for any $\alpha>0$,
	\[
		\prob{\N{M_n^{(n)}-M_0^{(n)}}_2\ge \alpha}\le 2d\exp(-\frac{\alpha^2}{2\sigma^2}),
	\]
	where
	\begin{align*}
		\sigma^2&=\sum_{s=1}^{{n\wedge (\Nout{\kappa }-1)}}\N*{\left(\prod_{r=s+1}^{n}\opL_r\right) D^{(s)}}_2^2
		\\&\le\sum_{s=1}^{{n\wedge (\Nout{\kappa }-1)}}\left[4[\nu^{1/2}\lambda_{1\sim p}(1+\kappa^2)+\nu_1\lambda_{1\sim p}\kappa]\mu\eta_s\right]^2 \prod_{r=s+1}^{n}\N*{\opL_r}_2^2
		\\&\le 32\left[\nu\lambda_{1\sim p}^2(1+\kappa^2)^2+\nu_1^2\lambda_{1\sim p}^2\kappa^2\right]\mu^2F_{D,2,2}^{(1,n)} 
		\\&\le 
		32C_{D,2,2}\left[\nu\lambda_{1\sim p}^2(1+\kappa^2)^2+\nu_1^2\lambda_{1\sim p}^2\kappa^2\right]\mu^2\eta_n\gamma^{-1}
.
	\end{align*}
Thus, noticing $T_D^{(n)}=M_n^{(n)}-M_0^{(n)}$ for $n\le \Nout{\kappa }-1$, we have
	\[
		\prob{\N{T_D^{(n)}}_2 \ge \alpha}\le 2d\exp\left(-\frac{\alpha^2}{
		32C_{D,2,2}\left[\nu\lambda_{1\sim p}^2(1+\kappa^2)^2+\nu_1^2\lambda_{1\sim p}^2\kappa^2\right]\mu^2\eta_n\gamma^{-1}
} \right).
	\]
	Choosing $\alpha=
		2C_{D,2,1}\lambda_1\lambda_{1\sim p}(1+\kappa^2)\mu^{3/2}\eta_n^{1/2}\gamma^{-3/2}
	$ and  noticing
	$T^{(n)}-T_*^{(n)}=T_D^{(n)}+T_R^{(n)}$
	and
	$\N{T_R^{(n)}}_2 \le
		2C_{D,2,1}\lambda_1\lambda_{1\sim p}(1+\kappa^2)\mu^{3/2}\eta_n^{1/2}\gamma^{-3/2}
	$
	, we have
	\begin{align*}
		\prob{\eventM_n(\chi)^{\rmc}}
		&= \prob{\N{T^{(n)}-T_*^{(n)}}_2  \ge
		4C_{D,2,1}\lambda_1\lambda_{1\sim p}(1+\kappa^2)\mu^{3/2}\eta_n^{1/2}\gamma^{-3/2}
		} 
		\\ &\le\prob{\N{T_D^{(n)}}_2  \ge
		2C_{D,2,1}\lambda_1\lambda_{1\sim p}(1+\kappa^2)\mu^{3/2}\eta_n^{1/2}\gamma^{-3/2}
	}
			\\ &\le 2d\exp\left(-\frac{[
		2C_{D,2,1}\lambda_1\lambda_{1\sim p}(1+\kappa^2)\mu^{3/2}\eta_n^{1/2}\gamma^{-3/2}
			]^2}{
		32C_{D,2,2}\left[\nu\lambda_{1\sim p}^2(1+\kappa^2)^2+\nu_1^2\lambda_{1\sim p}^2\kappa^2\right]\mu^2\eta_n\gamma^{-1}
	}\right)
	\\ &\le 2d\exp\left(-\frac{C_{D,2,1}^2\lambda_1^2\mu\gamma^{-2}}{8C_{D,2,2}\nu}\right)
	\\ &= 2d\exp(-C_M\lambda_1^2\nu^{-1}\gamma^{-2}\mu ),
		\end{align*}
	where $C_M =\frac{C_{D,2,1}^2}{8C_{D,2,2}}$.
\end{proof}
\subsection{Proof of \texorpdfstring{\Cref{lm:prop5:ppr1}}{Lemma~3.2}}\label{ssec:proof-of-lm:prop5:ppr1}

%\begin{proof}[Proof of \Cref{lm:prop5:ppr1}]
	First consider the case that \cref{eq:beta-gamma-kappa-mu} holds.
	We know $\eventT_r(\kappa)=\set*{\N{T^{(r)}}_2\le \kappa}$.
	If $n\ge\Nout{\kappa }$, then there exists some $r\le n$, such that $V^{(r)}\notin\sphere(\kappa )$, i.e.,
	$\N{T^{(r)}}_2> \kappa$ by \cref{eq:sphereK-normT}.
	Thus,
	\[
		\set{n\ge\Nout{\kappa }}\subset\bigcup_{r\le n}\set*{\N{T^{(r)}}_2>\kappa}\subset\bigcup_{r\le n}\eventT_r(\kappa)^{\rmc}.
	\]
	On the other hand, 
	for 
	$V^{(0)}\in\sphere(\kappa /2)$, %$\upsilon(1+\kappa^2)\mu^2\eta_o^{1/2}\gamma^{1/2}\le \kappa$,
$\eventM_r(\kappa,\mu)\subset\eventT_r(\kappa)$ because
\begin{equation*}\label{eq:Mn-subset-Tn}
	\begin{aligned}[b]
				\N{T^{(r)}}_2
				&\le\N{T^{(r)}-T_*^{(r)}}_2 +F_*^{(1,r)}\N{T^{(0)}}_2\\
				&\le\frac{1}{2}\upsilon(1+\kappa^2)\mu^{3/2}\eta_r^{1/2}\gamma^{1/2}+1\cdot\kappa/2
		\le \kappa.
	\end{aligned}
\end{equation*}
	Therefore,
	\[
		\bigcap_{r\le n}\eventM_r(\kappa,\mu)
		\subset \bigcap_{r\le n}\eventT_r(\kappa)
		\subset \set{n\le\Nout{\kappa }-1},
	\]
	and so
	\begin{equation}\label{eq:M_nH_n}
				\bigcap_{r\le {n\wedge (\Nout{\kappa }-1})}\eventM_r(\kappa)
					=\bigcap_{r\le n}\eventM_r(\kappa) 
					\subset\bigcap_{r\le n}\eventT_r(\kappa) 
		            =:\eventH_n.
	\end{equation}
	For $r\ge N_1:=N_{(\ln\varepsilon -\ln\kappa)/\ln(\eta_*\gamma)}^{(1)}$ and
	$V^{(0)}\in\sphere(\kappa/2)$,
	$\eventM_r(\kappa,\mu)\subset\eventT_r(\varepsilon )$
	because
	\begin{equation*}\label{eq:MinT}
		\begin{aligned}[b]
	\N{T^{(r)}}_2
	&\le\N{T^{(r)}-T_*^{(r)}}_2 +F_*^{(1,r)}\N{T^{(0)}}_2
	\\ &\le\frac{1}{2}\upsilon(1+\kappa^2)\mu^{3/2}\eta_r^{1/2}\gamma^{1/2}
	+(\eta_*\gamma)^{(\ln\varepsilon -\ln\kappa)/\ln(\eta_*\gamma)}\kappa/2
	\\ &\le\frac{1}{2}\kappa\rho_r^{1/2}
	+\varepsilon /2
	\qquad\qquad\text{{{by $\upsilon(1+\kappa^2)\mu^{3/2}\eta_o^{1/2}\gamma^{1/2}\le\kappa$}}}
	\\&\le\varepsilon ,
	\qquad\qquad\qquad\qquad\text{{{by $\kappa\rho_r^{1/2}\le\varepsilon $}}}
		.
		\end{aligned}
	\end{equation*}
	Therefore, 
	\[
		\bigcap_{r\le n}\eventM_{r}(\kappa)\subset\eventM_{N_1}(\kappa)\subset\set*{\Nin{\varepsilon }\le N_1}=:\eventH_o.
	\]
	Since
	\[
		\bigcap_{r\le{n\wedge (\Nin{\varepsilon }-1)}}\eventM_r(\kappa)
		\cap \eventH_o^{\rmc}
		\subset \bigcap_{r\le n}\eventM_r(\kappa)
		%\subset \eventM_{N_1}(\kappa)
		\subset \eventH_o,
	\]
	we have
	\[
		\bigcap_{r\le {n\wedge(\Nin{\varepsilon }-1)}}\eventM_r(\kappa)
		\subset \eventH_o.
	\]
	Together with \cref{eq:M_nH_n},
	\[
		\bigcap_{r\le {n\wedge(\Nin{\varepsilon }-1)\wedge(\Nout{\kappa}-1)}}\eventM_r(\kappa)
		\subset \eventH_o\cap\eventH_{n}.
	\]
	By \Cref{lm:prop3:ppr1}, we get
	\begin{align*}
		\MoveEqLeft[4]\prob{\bigcup_{r\le {n\wedge (\Nin{\varepsilon }-1)\wedge(\Nout{\kappa}-1)}}\eventM_r(1)^{\rmc}\cap\eventQ_{n}(\mu)}
	\\ &\le \left({n\wedge (\Nin{\varepsilon }-1)\wedge(\Nout{\kappa}-1)}\right) 2d\exp(-C_M\gamma^{-2}\lambda_{1\sim p}^{2} \mu^2)
		\\ &= 2nd\exp(-C_M\gamma^{-2}\lambda_{1\sim p}^{2} \mu^2)
		.
	\end{align*}
	Thus, 
		\begin{align*}
		\prob{(\eventH_{n}\cap\eventH_o)^{\rmc}}
		&\le\prob{(\eventH_{n}\cap\eventH_o\cap\eventQ_{n}(\mu))^{\rmc}}
		\\&=\prob{(\eventH_{n}\cap\eventH_o)^{\rmc}\cup\eventQ_{n}(\mu)^{\rmc}}
				\\ &=\prob{(\eventH_{n}\cap\eventH_o)^{\rmc}\cap\eventQ_{n}(\mu)}
				+\prob{\eventQ_{n}(\mu)^{\rmc}}
				\\ &\le\prob{\bigcup_{r\le {n\wedge (\Nin{\varepsilon }-1)\wedge(\Nout{\kappa}-1)}}\eventM_r(\kappa)^{\rmc}\cap\eventQ_{n}(\mu)}
				+\prob{\eventQ_{n}(\mu)^{\rmc}}
				\\ &\le 2nd\exp(-C_M  \lambda_1^2\nu^{-1}\gamma^{-2}\mu)
				+n(\ee d+p+1)\exp\left(-C_{\psi}\psi^{-1}(1\wedge\psi^{-1})\mu\right)
				,
		\end{align*}
		where $\prob{\eventQ_{n}(\mu)^{\rmc}}$ is given by \Cref{lm:quasi-bounded}.
%\end{proof}

		\subsection{Proof of \texorpdfstring{\Cref{lm:prop4:ppr1}}{Lemma~3.3}}\label{ssec:proof-of-lm:prop4:ppr1}
%\begin{proof}[Proof of \Cref{lm:prop4:ppr1}]
	Choose 
	\[
		\eventH=\wtd\eventH_n\cap\eventQ_n=\bigcap_{r\in[1,N_1-1]}\eventT_r(2)\cap\bigcap_{r\in [N_1,n]} \eventT_r(2\varepsilon)\cap\eventQ_n(\mu).
	\]
	By \Cref{lm:prop5:ppr1} with $\kappa=2, \mu=\left(\frac{\psi\vee\psi^2}{C_\psi}\vee\frac{\nu_1^2\gamma^2}{C_M\lambda_1^2}\right)\ln\frac{(\ee+3)nd}{\delta_1}$,
	\[
		\prob{\eventH^{\rmc}}
		%\prob{(\wtd\eventH_{n}\cap\eventH_o)^{\rmc}}
		\le 2nd\exp(-C_M  \lambda_1^2\nu^{-1}\gamma^{-2}\mu)
		+n(\ee d+p+1)\exp\left(-C_{\psi}\psi^{-1}(1\wedge\psi^{-1})\mu\right)
		\le \delta_1
		.
	\]
	In the following, we will work on the expectation.

	Note that 
	\begin{align*}
		\wtd\eventH_n
		&=\bigcap_{r\in[1,N_1-1]}\eventT_r(2)\cap\bigcap_{r\in [N_1,n]} \eventT_r(2\varepsilon)
		\subset\bigcap_{r\in[1,n]}\set*{\ind{\eventT_{r-1}(2)}D^{(r)}=D^{(r)}}
		%\\&\subset\bigcap_{r\in[1,N_1-1]}\set*{\ind{\eventT_{r-1}(2)}D^{(r)}=D^{(r)}}\cap\bigcap_{r\in [N_1,n]} \eventT_r(2\varepsilon)\set*{\ind{\eventT_{r-1}(2\varepsilon)}D^{(r)}=D^{(r)}},
	\end{align*}
	we have for $n\ge N_1$,
	\begin{align*}
		T^{(n)}\ind{\wtd\eventH_n\cap\eventQ_n}
		&= \left( \prod_{r=1}^{n} \opL_r\right) T^{(0)}\ind{\eventQ_n}
		+\sum_{s=1}^{n}\left(\prod_{r=s+1}^{n} \opL_r\right) D^{(s)}\ind{\eventT_{s-1}(2)\cap\eventQ_n}
		%\\&\qquad
		%+\sum_{s=N_1}^{n}\left(\prod_{r=s+1}^{n} \opL_r\right) D^{(s)}\ind{\eventT_{s-1}(2\varepsilon)\cap\eventQ_n}
		+\sum_{s=1}^{n} \left(\prod_{r=s+1}^{n} R_E^{(s-1)} \right)\ind{\eventQ_n}
		\\ &=: \wtd T_*^{(n)}+\wtd T_{D}^{(n)}+\wtd T_R^{(n)}.
	\end{align*}
	%In what follows, we simply write $R_E^{(n)}=R_E^{(n)}$ for convenience.
	Then, %by \cref{eq:I1I2I3},
	\begin{align*}
		\MoveEqLeft[3]\E{T^{(n)}\circ T^{(n)}\overevent \wtd\eventH_n\cap\eventQ_n}
	\\	&= \E{T^{(n)}\circ T^{(n)}\ind{\wtd\eventH_n\cap\eventQ_n}}		\\
	&=\E{\wtd T_*^{(n)}\circ \wtd T_*^{(n)}}+2\E{\wtd T_*^{(n)}\circ \wtd T_{D}^{(n)}}+2\E{\wtd T_*^{(n)}\circ \wtd T_R^{(n)}}+\E{\wtd T_R^{(n)}\circ \wtd T_R^{(n)}}			\\
        %&\qquad+\E{[\wtd T_{D,1}^{(n)}+\wtd T_{D,2}^{(n)}]\circ [\wtd T_{D,1}^{(n)}+\wtd T_{D,2}^{(n)}]}	
        &\qquad+\E{\wtd T_{D}^{(n)}\circ \wtd T_{D}^{(n)}}	
		+2\E{\wtd T_{D}^{(n)}\circ \wtd T_R^{(n)}} 	\\
		&\le\E{\wtd T_*^{(n)}\circ \wtd T_*^{(n)}}+2\E{\wtd T_*^{(n)}\circ \wtd T_{D}^{(n)}} 
		+2\E{\wtd T_*^{(n)}\circ \wtd T_R^{(n)}}
		+2\E{\wtd T_R^{(n)}\circ \wtd T_R^{(n)}} 
		+2\E{\wtd T_{D}^{(n)}\circ \wtd T_{D}^{(n)}}.
	\end{align*}
	In the following, we estimate each summand above for $n\in[N_2,K]$. 
	\begin{enumerate}
		\item
			$\dps\E{\wtd T_*^{(n)}\circ \wtd T_*^{(n)}}=\left(\prod_{r=1}^{n} \opL_r^{2}\right) T^{(0)}\circ T^{(0)}$.
		\item
			$\dps\E{\wtd T_*^{(n)}\circ \wtd T_{D}^{(n)}}=\left(\prod_{r=1}^{n} \opL_r\right) \sum_{s=1}^{n}\left(\prod_{r=s+1}^{n} \opL_r\right)  T^{(0)}\circ\E{D^{(s)}\ind{\eventT_{s-1}(2)}\ind{\eventQ_n}}=0$, because
			$\eventT_{s-1}(2)\subset\fil_{s-1}$ and so
			\begin{align*}
				\E{D^{(s)}\ind{\eventT_{s-1}(2)}\ind{\eventQ_n}}
				&=\prob{\eventT_{s-1}(2)}\E{D^{(s)}\ind{\eventQ_n}\given\eventT_{s-1}(2)}
				\\&=\prob{\eventT_{s-1}(2)}\E{\E{D^{(s)}\ind{\eventQ_n}\given\fil_{s-1}}\given\eventT_{s-1}(2)}
				=0.
			\end{align*}
			%\item 
				%$\dps\E{\wtd T_*^{(n)}\circ \wtd T_{D,2}^{(n)}}=\left(\prod_{r=1}^{n} \opL_r\right) \sum_{s=N_1}^{n}\left(\prod_{r=s+1}^{n} \opL_r\right)  T^{(0)}\circ\E{D^{(s)}\ind{\eventT_{s-1}(2\varepsilon)}\ind{\eventQ_n}}=0$ similarly.
		\item
			$\dps\E{\wtd T_*^{(n)}\circ \wtd T_R^{(n)}}=\left(\prod_{r=1}^{n} \opL_r\right) \sum_{s=1}^{n} \left(\prod_{r=s+1}^{n} \opL_r\right) T^{(0)}\circ\E{R_E^{(s-1)}\ind{\eventQ_n}}$.
			Recall \cref{eq:schur-ineq-UI}. By \cref{itm:lm:diff-T:dT} of \Cref{lm:diff-T}, we have
			\begin{align*}
				\N{\E{\wtd T_*^{(n)}\circ \wtd T_R^{(n)}}}_2
				&\le F_*^{(1,n)} F_{D,2,1}^{(1,n)} 2\lambda_1\lambda_{1\sim p}\mu\kappa(1+\kappa^2)^{1/2}
				\\&\le \varepsilon C_{D,2,1}\eta_n\gamma^{-1}\lambda_1\lambda_{1\sim p}\kappa(1+\kappa^2)^{1/2}\mu
				\qquad\qquad\text{by $n\ge N_1 $}
				\\&=  \frac{1}{8}\upsilon \kappa(1+\kappa^2)^{1/2}\mu\eta_n\gamma\varepsilon
				\qquad\qquad\qquad\qquad\qquad\text{by $\upsilon=8C_{D,2,1}\lambda_1\lambda_{1\sim p}\gamma^{-2}$}
				\\&\le  \frac{1}{8} \kappa^2(1+\kappa^2)^{-1/2}\mu^{-1/2}\rho_n\eta_o^{1/2}\gamma^{1/2}\varepsilon
				\qquad\qquad\text{by $\upsilon(1+\kappa^2)\mu^{3/2}\eta_o^{1/2}\gamma^{1/2}\le\kappa$}
				\\&\le  \frac{1}{2\sqrt{5}} \mu^{-1/2}\rho_n\eta_o^{1/2}\gamma^{1/2}\varepsilon
				\qquad\qquad\qquad\qquad\text{by $\kappa=2$}
				%\\&=  \frac{1}{8} (1+\kappa^2)^{1/2}\rho_n^{1/2}\eta_o^{1/2}\gamma^{1/2}\varepsilon^2
				%\\&=  \frac{1}{4\sqrt{2}} \eta_n^{1/2}\gamma^{1/2}\varepsilon^2
				%\qquad\qquad\qquad\qquad\qquad\text{by $\kappa\rho_n^{1/2}\le\varepsilon$}
				.
			\end{align*}
		\item
			$\dps\E{\wtd T_R^{(n)}\circ \wtd T_R^{(n)}}=\sum_{s=1}^{n}\left(\prod_{r=s+1}^{n} \opL_r^2\right) \E{R_E^{(s-1)}\ind{\eventQ_n}\circ R_E^{(s-1)}\ind{\eventQ_n}}$.
			Also, by \cref{eq:schur-ineq-UI},
			\begin{align*}
				\N{\E{\wtd T_R^{(n)}\circ \wtd T_R^{(n)}}}_2
				&\le F_{D,4,2}^{(1,n)}[ 2\lambda_1\lambda_{1\sim p}\mu\kappa(1+\kappa^2)^{1/2}]^2
				\\&\le C_{D,4,2}\eta_n^3\gamma^{-1}4\lambda_1^2\lambda_{1\sim p}^2\mu^2\kappa^2(1+\kappa^2)
				\\&= \frac{C_{D,4,2}}{16C_{D,2,1}^2}\upsilon^2\kappa^2(1+\kappa^2)\mu^2\eta_n^3\gamma
				\\&\le \frac{C_{D,4,2}}{16C_{D,2,1}^2}\kappa^4(1+\kappa^2)^{-1}\mu^{-1}\rho_n\eta_n^2
				\\&\le \frac{C_{D,4,2}}{5C_{D,2,1}^2}\mu^{-1}\rho_n\eta_n^2
				%\\&\le \frac{C_{D,4,2}}{16C_{D,2,1}^2}(1+\kappa^2)\eta_n^2\gamma^2\varepsilon^2
				%\\&\le \frac{C_{D,4,2}}{8C_{D,2,1}^2}\eta_o\eta_n\gamma^2\varepsilon^4
				%\qquad\qquad\qquad\qquad\qquad\text{by $\kappa\ge1$}
				.
			\end{align*}
		%\item
			%$\dps\E{\wtd T_{D,1}^{(n)}\circ \wtd T_{D,2}^{(n)}}=\sum_{s=1}^{N_1-1}\left(\prod_{r=s+1}^{n} \opL_r\right) \sum_{s'=N_1}^{n}\left(\prod_{r=s'+1}^{n} \opL_r\right)  \E{D^{(s)}\ind{\eventQ_n}\circ D^{(s')}\ind{\eventT_{s'-1}(2)}\ind{\eventQ_n}}=0$,
			%because $s<s'$ and
			%\begin{align*}
				%\E{D^{(s)}\ind{\eventQ_n}\circ D^{(s')}\ind{\eventT_{s'-1}(2)}\ind{\eventQ_n}}
				%&= \E{D^{(s)}\circ D^{(s')}\ind{\eventT_{s'-1}(2)}\ind{\eventQ_n}}
				%\\ &= \prob{\eventT_{s'-1}(2)}\E{D^{(s)}\circ D^{(s')}\ind{\eventQ_n}\given\eventT_{s'-1}(2)}
				%\\ &= \prob{\eventT_{s'-1}(2)}\E{\E{D^{(s)}\circ D^{(s')}\ind{\eventQ_n}\given\fil_{s'-1}}\given\eventT_{s'-1}(2)}
				%\\ &= \prob{\eventT_{s'-1}(2)}\E{\E{D^{(s')}\ind{\eventQ_n}\given\fil_{s'-1}}\circ D^{(s)}\given\eventT_{s'-1}(2)} 
				%=0.
			%\end{align*}
		\item For $\dps\E{\wtd T_{D}^{(n)}\circ \wtd T_{D}^{(n)}}$, we have
            \begin{align*}
				\dps\E{\wtd T_{D}^{(n)}\circ \wtd T_{D}^{(n)}}&=\sum_{s=1}^{n}\left(\prod_{r=s+1}^{n} \opL_r^2\right) \E{D^{(s)}\ind{\eventQ_n}\ind{\eventT_{s-1}(2)}\circ D^{(s)}\ind{\eventQ_n}\ind{\eventT_{s-1}(2)}} \\
				&=\sum_{s=1}^{n}\eta_s^2\left(\prod_{r=s+1}^{n} \opL_r^2\right)H_{\circ} + \sum_{s=1}^{n}\left(\prod_{r=s+1}^{n} \opL_r^2\right)\E{R_{\circ}^{(s)}\ind{\eventT_{s-1}(2)}},
			\end{align*}
			because for $s\ne s'$,
			\begin{align*}
				\MoveEqLeft[4] \E{D^{(s)}\ind{\eventQ_n}\ind{\eventT_{s-1}(2)}\circ D^{(s')}\ind{\eventQ_n}\ind{\eventT_{s'-1}(2)}}
				=\E{D^{(s)}\circ D^{(s')}\ind{\eventQ_n}\ind{\eventT_{s-1}(2)}\ind{\eventT_{s'-1}(2)}} \\
				&=\prob{\eventT_{s-1}(2)\cap\eventT_{s'-1}(2)}\E{D^{(s)}\circ D^{(s')}\ind{\eventQ_n}\given\eventT_{s-1}(2)\cap\eventT_{s'-1}(2)} \\
				&=\prob{\eventT_{s-1}(2)\cap\eventT_{s'-1}(2)}\E{\E{D^{({s\vee s'})}\ind{\eventQ_n}\given\fil_{{s\vee s'}-1}}
				\circ D^{({s\wedge s'})}\given\eventT_{s-1}(2)\cap\eventT_{s'-1}(2)} \\
                &=0,
                \end{align*}
and
			\begin{align*}
				\E{D^{(s)}\ind{\eventQ_n}\ind{\eventT_{s-1}(2)}\circ D^{(s)}\ind{\eventQ_n}\ind{\eventT_{s-1}(2)}}
				&=\E{D^{(s)}\circ D^{(s')}\ind{\eventQ_n}\ind{\eventT_{s-1}(2)}}
				\\ &=\prob{\eventT_{s-1}(2)}\E{\E{D^{(s)}\circ D^{(s)}\ind{\eventQ_n}\given\fil_{s-1}}\given\eventT_{s-1}(2)}
				\\ &\le\eta_s^2H_{\circ}+\E{R_{\circ}^{(s)}\ind{\eventT_{s-1}(2)}}.
			\end{align*}
			On the event $\eventT_{s-1}(2)$, by $\lambda_{1\sim p}\mu\eta_o(1+\eta_o)<1$, we have
			\begin{align*}
				\N{R_{\circ}^{(s)}\ind{\eventT_{s-1}(2)}}_2
				&\le
				2\nu_1\nu^{1/2}\lambda_{1\sim p}^2\mu^2\eta_s^2\tau_{s-1} \Big(1+\left[1+\nu_1\nu^{-1/2}\right]\tau_{s-1} +\tau_{s-1} ^2+\frac{1}{2}\tau_{s-1} ^3\Big)
				\\ &\qquad\qquad + 8\nu\lambda_{1\sim p}^3\mu^3\eta_s^3(1+\tau_{s-1}^2)^{1/2} \left[1+\tau_{s-1}^2+\nu_1\nu^{-1/2}\tau_{s-1}\right]^2
\\				&\le
8\nu_1\nu^{1/2}\lambda_{1\sim p}^2\mu^2\eta_s^2 \left[5+\nu_1\nu^{-1/2}\right]
+ 8\sqrt{5}\nu\lambda_{1\sim p}^3\mu^3\eta_s^3 \left[5+2\nu_1\nu^{-1/2}\right]^2
\\				&\le
48\nu_1^2\lambda_{1\sim p}^2\mu^2\eta_s^2 + 392\sqrt{5}\nu_1^2\lambda_{1\sim p}^3\mu^3\eta_s^3 
				.
			\end{align*}
			Thus,
			\begin{align*}
				\MoveEqLeft[4]	\N*{\sum_{s=1}^{n}\left(\prod_{r=s+1}^{n} \opL_r^2\right)\E{R_{\circ}^{(s)}\ind{\eventT_{s-1}(2)}}}_2
				\\ &\le\sum_{s=1}^{n}\left(\prod_{r=s+1}^{n} \N{\opL_r}^2\right) 
				\left(48\nu_1^2\lambda_{1\sim p}^2\mu^2\eta_s^2 + 392\sqrt{5}\nu_1^2\lambda_{1\sim p}^3\mu^3\eta_s^3 \right)
				%\\ &\le\left(\prod_{r=N_1}^{n} \N{\opL_r}^2\right)\sum_{s=1}^{N_1-1}\left(\prod_{r=s+1}^{N_1-1} \N{\opL_r}^2\right) 21\lambda_{1\sim d}^2\mu^2\eta_{s}^2 + \wtd R_{D}
				%\\&\qquad\qquad\qquad\text{write $\wtd R_{D}=\sum_{s=1}^{N_1-1}\left(\prod_{r=s+1}^{n} \N{\opL_r}^2\right) 250\sqrt{5}\lambda_{1\sim d}\lambda_{1\sim p}^2\mu^3 \eta_{s} ^3$} 
				%\\ &\le \left(F_*^{(N_1,n)}\right)^2F_{D,2,2}^{(1,N_1-1)}  21\lambda_{1\sim d}^2\mu^2 + \wtd R_{D}
				%%\\ &\le \left(\frac{F_*^{(1,n)}}{F_*^{(1,N_1-1)}}\right)^2C_{D,2,2}\eta_{N_1-1}\gamma^{-1}  21\lambda_{1\sim d}^2\mu^2 + \wtd R_{D}
				%%\\ &\le \left(\frac{\varepsilon/\kappa^2}{\varepsilon/\kappa}\right)^2 21C_{D,2,2} \lambda_{1\sim d}^2\mu^2\eta_{N_1-1}\gamma^{-1}  + \wtd R_{D}
			%%\\&\qquad\qquad\qquad\text{by $n\ge N_2$ and $F_*^{(1,N_1)}\le%(\eta_*\gamma)^{-\ln\kappa/\ln(\eta_*\gamma)}=
				%%\frac{\varepsilon}{\kappa}< F_*^{(1,N_1-1)}$ by definition}
				%\\ &\le  21C_{D,2,2} \lambda_{1\sim d}^2\mu^2\eta_{N_1-1}\gamma^{-1}  + \wtd R_{D}
				%\\ &\le 21C_{D,2,2} \lambda_{1\sim d}^2\mu^2\eta_{N_1-1}\gamma^{-1}  + \wtd R_{D}
				\\ &\le
				F_{D,2,2}^{(1,n)}48\nu_1^2\lambda_{1\sim p}^2\mu^2 
				+ F_{D,3,2}^{(1,n)}392\sqrt{5}\nu_1^2\lambda_{1\sim p}^3\mu^3
				\\ &\le
				48C_{D,2,2}\nu_1^2\lambda_{1\sim p}^2\mu^2  \eta_n\gamma^{-1}
				+ 392\sqrt{5}C_{D,3,2}\nu_1^2\lambda_{1\sim p}^3\mu^3 \eta_n^2\gamma^{-1}
				\\ &\le
				\frac{48C_{D,2,2}}{400C_{D,2,1}^2}\nu_1^2\lambda_1^{-2}\mu^{-1}  \rho_n\gamma^{2}
				+ \frac{392\sqrt{5}C_{D,3,2}}{8000C_{D,2,1}^3}\nu_1^2\lambda_1^{-3}\mu^{-3/2}\eta_o^{1/2} \rho_n^2\gamma^{7/2}
		.
	\end{align*}
	\end{enumerate}
	Collecting all estimates together, we obtain
	\begin{align*}
		\E{T^{(n)}\circ T^{(n)}\overevent \wtd\eventH_n\cap\eventQ_n}
		&\le\left(\prod_{r=1}^{n} \opL_r^2\right)T^{(0)}\circ T^{(0)} +2\sum_{s=1}^{n}\eta_s^2\left(\prod_{r=s+1}^{n} \opL_r^2\right)H_{\circ} +R
		,
	\end{align*}
	where%, by \cref{itm:beta-C_Vdelta}, $2C_{\Delta} \kappa \lambda_{1\sim p}\mu\eta_*^{1/2}\le 1$, and
	\begin{align*}
		\N{R}_2
		&\le
		\begin{multlined}[t]
		\frac{1}{\sqrt{5}} \mu^{-1/2}\rho_n\eta_o^{1/2}\gamma^{1/2}\varepsilon
		+\frac{2C_{D,4,2}}{5C_{D,2,1}^2}\mu^{-1}\rho_n\eta_n^2
		+		\frac{6C_{D,2,2}}{25C_{D,2,1}^2}\nu_1^2\lambda_1^{-2}\mu^{-1}  \rho_n\gamma^{2}
		\\		+ \frac{49\sqrt{5}C_{D,3,2}}{1000C_{D,2,1}^3}\nu_1^2\lambda_1^{-3}\mu^{-3/2}\eta_o^{1/2} \rho_n^2\gamma^{7/2}
		\end{multlined}
		\\&\le
		\frac{1}{4\sqrt{5}}\mu^{-1}\varepsilon^3
		+\frac{C_{D,4,2}}{160C_{D,2,1}^2}\mu^{-1}\eta_o^2\varepsilon^6
		+		\frac{3C_{D,2,2}}{50C_{D,2,1}^2}\nu_1^2\lambda_1^{-2}\mu^{-1}  \gamma^{2}\varepsilon^2
				+ \frac{49\sqrt{5}C_{D,3,2}}{16000C_{D,2,1}^3}\nu_1^2\lambda_1^{-3}\mu^{-3/2}\eta_o^{1/2} \gamma^{7/2}\varepsilon^4
				\\&\le C'\nu_1^2\lambda_1^{-2}\mu^{-1}\gamma^2\varepsilon^2
				\le C_R\frac{\varepsilon^2}{\ln\frac{nd}{\delta_1}}
,
	\end{align*}
	in which $C'$ and $C_R$ are absolute constants.
%\end{proof}

\section{Conclusion}\label{sec:conclusion}
We have presented a convergence analysis of the Oja's method for online/streaming PCA iteration with sub-Gaussian samples, by combining the idea in Li et al.~\cite{liWLZ2017near} and Liang et al.~\cite{liangGLL2017nearly:arxiv} and the convergence result from a random guess to a good guess by Huang et al.~\cite{huang2021streaming}.
Our results show for the first time that the Oja' method for online PCA is \emph{optimal} in the sense that with high probability the convergence rate exactly matches the minimax information lower bound for offline PCA.

Recently \cite{amid2019implicit,henriksen2019adaoja} developed a method to decide the learning rate adaptively.
Though our framework works for different strategies on choosing the learning rates, it remains an open but very worthwhile problem to consider how fast the adaptive learning rates would accelerate the convergence,
because with a small probability to cover bad events, it is quite possible to attain a convergence rate even below the lower bound.

%\clearpage
{\small
	\bibliographystyle{plain}
	\bibliography{../strings,../liang-nctu,../liang-tsinghua}
}

\end{document}